%% file: main.tex
\documentclass[]{interact}

\usepackage{epstopdf}
\usepackage[caption=false]{subfig}

\usepackage[natbibapa,nodoi]{apacite} 
\theoremstyle{plain}
\newtheorem{property}{Property}[section]
\newtheorem{theorem}{Theorem}[section]
\newtheorem{lemma}[theorem]{Lemma}

\theoremstyle{definition}

\theoremstyle{remark}

\usepackage{tabularx}
\usepackage{xcolor}
\usepackage{etoolbox} 
\usepackage{graphicx}
\usepackage{hyperref}
\usepackage{multirow}
\usepackage{lscape}

\usepackage[T1]{fontenc} 
\usepackage[utf8]{inputenc} 

\usepackage{algorithm}
\usepackage{algpseudocode}

\newtoggle{anonymous}
\togglefalse{anonymous}

\newcommand{\dominanceindex}{\bar{S}}
\newcommand \multinomial {\mathcal{M}(F, N)}
\newcommand \averageswapdistance {\left< d \right>}
\newcommand \nmax {7}

\begin{document}

\title{Swap distance minimization beyond entropy minimization in word order variation}

\iftoggle{anonymous}{}
{
  \author{
  \name{Víctor Franco-Sánchez\textsuperscript{a}, Arnau Martí-Llobet\textsuperscript{b} and Ramon Ferrer-i-Cancho\textsuperscript{c}\thanks{CONTACT Ramon Ferrer-i-Cancho. Email: rferrericancho@cs.upc.edu}}
  \affil{\textsuperscript{a}Departament de Matemàtiques, Universitat Polit\`ecnica de Catalunya (UPC), Barcelona, Spain. ORCiD: 0009-0007-4163-2476;
  \textsuperscript{b}Universitat Polit\`ecnica de Catalunya (UPC), Barcelona, Spain; 
  \textsuperscript{c}Quantitative, Mathematical and Computational Linguistics Research Group, Departament de Ci\`encies de la Computaci\'o, Universitat Polit\`ecnica de Catalunya (UPC), Barcelona, Spain. ORCiD: 0000-0002-7820-923X}
  }
}

\maketitle
    
\begin{abstract}
    Consider a linguistic structure formed by $n$ elements, for instance, subject, direct object and verb ($n=3$) or subject, direct object, indirect object and verb ($n=4$). We investigate whether the frequency of the $n!$ possible orders is constrained by two principles. First, entropy minimization, a principle that has been suggested to shape natural communication systems at distinct levels of organization. Second, swap distance minimization, i.e. a preference for word orders that require fewer swaps of adjacent elements to be produced from a source order. We present average swap distance, a novel score for research on swap distance minimization. 
    We find strong evidence of pressure for entropy minimization and swap distance minimization with respect to a die rolling experiment in distinct linguistic structures with $n=3$ or $n=4$.
    Evidence from a Polya urn process is strong for $n=4$ but weaker for $n=3$. We still find evidence consistent with the action of swap distance minimization when word order frequencies are shuffled, indicating that swap distance minimization effects go beyond pressure to reduce word order entropy.
\end{abstract}

\begin{keywords}
word order; entropy minimization; swap distance minimization
\end{keywords}



\section{Introduction}

\label{sec:introduction}

Languages employ syntactic structures to communicate. Some examples are the SOV structure, formed by a subject (S), an object (O) and a verb (V) \citep{wals-81,Hammarstroem2016a} and the nAND structure, a noun phrase consisting of a noun (n), an adjective (A), a numeral (N) and a demonstrative (D) \citep{Dryer2018a,Culbertson2020a}. DNAn is the typical order of such a noun phrase in English, as in ``These three black horses''.
A structure formed by $n$ elements has $n!$ possible orders ($n=3$ for SOV and $n=4$ for nAND). 
Here we investigate the constraints that operate on the $n!$ permutations of a syntactic structure as a statistical mechanics problem. First we review the principles that may constrain word order variation in that setting.  

{\em The principle of} entropy minimization has been argued to shape natural communication systems \citep{Ferrer2015b,Ferrer2013f}. 
The minimization of word entropy puts pressure to reduce the effective vocabulary size \citep{Ferrer2015b} and is one of the main ingredients of models that shed light on the origins of Zipf's law for word frequencies \citep{Ferrer2004e}, the tendency of more frequent words to be older in a language \citep{Casado2021b}, and vocabulary learning in children \citep{Ferrer2013g,Carrera2021a}. 
The principle of word entropy minimization is justified by word frequency effects, namely the higher mental accessibility of more frequent words (see \citet{Jescheniak1994a,Dahan2001a} and references therein). Such accessibility is maximized when only one word can be produced as then the frequency of the only word is maximum, that is, when entropy reaches its minimum value. When all words have the same probability (which is a small number if the number of word orders is sufficiently large), entropy yields its maximum value.  

The principle of word entropy minimization can be extended to blocks of elements that form syntactic structures.
We define $p_i$ as the probability of the $i$-th order of a structure of $n$ elements. Then, $H$, the entropy of the possible orders of a structure, is defined as
\begin{eqnarray*}
H = - \sum_{i=1}^{n!} p_i \log p_i.
\label{eq:entropy}
\end{eqnarray*}
Block entropy minimization is justified in terms of frequency effects on blocks of words: word combinations that appear more often are easier to learn and to process \citep{Contreras2022a}. This block entropy minimization predicts that using only one of the permutations is optimal. When applied to the SOV structure, it predicts that only one order, say SVO, has non-zero probability. This prediction matches to some degree
the linguistic notion of canonical or basic order, namely the usual order in a language under certain conditions \citep{Comrie1989a} 
or the suspected subjacent word order in case the usual order is not manifested \citep{Chomsky1965}. 
Here we aim to test whether the order of various kinds of syntactic structures is shaped by block entropy minimization. 

The principle of word entropy minimization has been extended to the linear order of words in the sentence by the notion of conditional entropy \citep{Ferrer2013f}. In simple terms, a word whose processing cost needs to be reduced should be preceded by words that reduce its uncertainty. This principle of conditional entropy minimization has been used to predict the optimal placement of head words with respect to their dependents. An example is the nAND structure. The principle of conditional entropy minimization predicts that the head should be placed first or last and this has been confirmed experimentally in conditions where that principle is less likely to suffer interference by other word order principles \citep{Ferrer2023b,Ferrer2019a}.

\begin{figure}
\caption{\label{fig:word_order_permutation_ring} The word order permutation ring for the structure SOV, that is an instantiation of the permutohedron of order $3$. The number below each order indicates the swap distance to SOV. }
\centering
\includegraphics[width=0.5\textwidth]{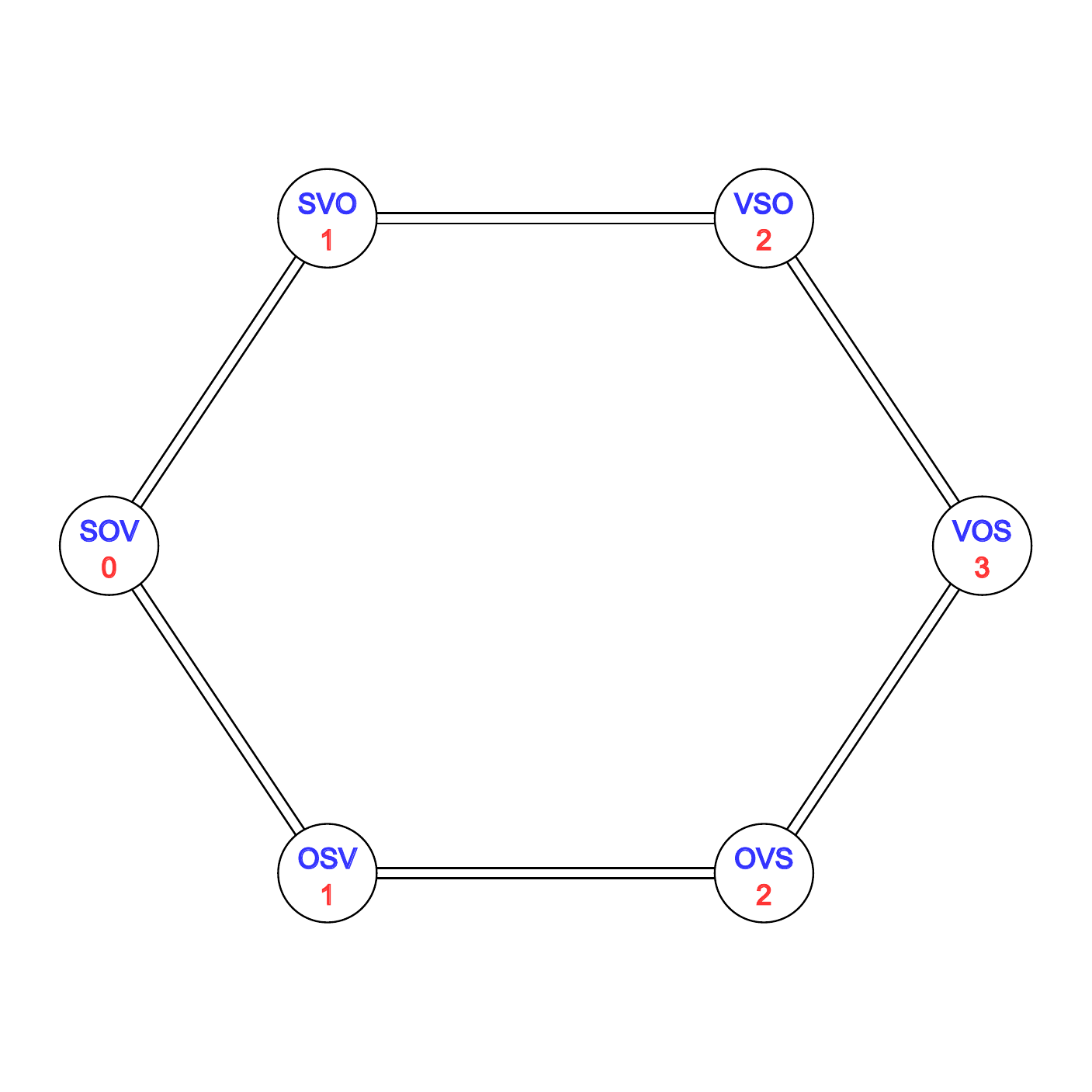}
\end{figure}

The principle of word entropy minimization predicts a tendency to reduce the effective size of the set of orders manifested but it is neutral concerning how orders other than the most frequent word order are going to be selected. Here we will confront entropy minimization over all permutations of a word order sequence against another word order principle that also predicts a preference for a canonical word order: swap distance minimization \citep{Ferrer2023a}. That principle states that given a certain word order (e.g., the canonical word order), variations that involve a smaller number of swaps of adjacent elements are cognitively easier \citep{Ferrer2023a}. Consider the case of the order of S, O and V. SOV requires just one swap of adjacent constituents to produce SVO (the swap of OV) or to produce OSV (the swap of SO), two swaps of adjacent constituents to become VSO or OVS and three swaps of adjacent constituents to produce VOS (Figure \ref{fig:word_order_permutation_ring}). It has been shown that the cognitive cost of processing an order of S, O and V is significantly correlated with its swap distance to the canonical order in languages whose canonical order is SOV \citep{Ferrer2023a}. The graph in Figure \ref{fig:word_order_permutation_ring} is an instance of the permutohedron, a graph where vertices are the orders of a syntactic structure and an edge joining two vertices indicates that one vertex produces the other vertex by swapping a couple of adjacent word orders. The swap distance between two orders is their distance in the permutohedron. 

Here we introduce a new score for measuring the effect of swap distance minimization when the source order is unknown, unclear, or there may be multiple source orders from which word order variation is produced. This novel score, that we call average swap distance, is defined as 
\begin{eqnarray}
\averageswapdistance = \sum_{i=1}^{n!} \sum_{j=1}^{n!} d_{ij} p_i p_j,
\label{eq:average_swap_distance}
\end{eqnarray}
where $d_{ij}$ is the swap distance between permutations $i$ and $j$.

The remainder of this article is organized as follows. The mathematically oriented reader is invited to check a series of technical appendices. 
In this article, we use the terms permutation and order (sequential order) interchangeably.   
A sequence of $n$ elements has $N=n!$ possible orderings. $N$ is also the number of vertices of the permutohedron.
The diversity indices are presented in greater mathematical detail in Appendix \ref{app:diversity_indices}. The graph theory and the combinatorics of the permutohedron are presented in Appendix \ref{app:theory}. 
The mathematical properties of $\averageswapdistance$ are analyzed in Appendix \ref{app:mathematical_properties_of_average_swap_distance}, focusing on its range of variation (lower and upper bounds). For the general reader, the essential result is that $0 \leq \averageswapdistance \leq \averageswapdistance_{max}$ where 
\begin{equation}
\averageswapdistance_{max} = \frac{n(n-1)}{4}
\label{eq:average_swap_distance_max_main_text}
\end{equation}
at least for $n \leq \nmax$.
$\averageswapdistance$ shares some mathematical properties with $H$: it is minimized when only one order is used and it is maximized when all words are equally likely.
The expectations of $H$ and $\averageswapdistance$ under distinct null hypotheses are investigated in Section \ref{app:null_hypotheses}. For the general reader, the essential results are presented next. 
The expectation of $\averageswapdistance$ under the null hypothesis of die rolling ($dr$), i.e.  rolling a fair die to produce the orders, is 
\begin{equation}
\averageswapdistance_{dr} =  \frac{F-1}{F} \frac{n(n-1)}{4},
\label{eq:expected_average_swap_distance_die_rolling_main_text}
\end{equation}
where $F$ is the total frequency of the word orders, that is   
\begin{equation}
F = \sum_{i=1}^{n!} f_i,
\label{eq:total_frequency}
\end{equation}
where $f_i$ is the frequency of the $i$-th word order.

The expected value of $\averageswapdistance$ under the null hypothesis of a random permutation ($rp$), i.e. shuffling the actual word order frequencies, is 
\begin{eqnarray}
\averageswapdistance_{rp} = \dominanceindex \frac{n!}{n!-1}\frac{n(n-1)}{4},
\label{eq:expected_average_swap_distance_random_permutation_main_text}
\end{eqnarray}
where $\dominanceindex = 1 - S$ is the so-called dominance index and $S$ is the Simpson index, that is defined as \citep{Sommerfield2008a}
\begin{eqnarray}
S= \sum_{i=1}^{n!} p_i^2.
\label{eq:Simpson_index}
\end{eqnarray}
In that setting, entropy remains constant. The random permutation null hypothesis aims to test for the presence of swap distance minimization beyond block entropy minimization taking into account that the frequency of a word order is likely to be determined by processing, learning and evolutionary constraints that are independent of swap distance minimization or even compete with it 
\citep{Levshina2023a,Ferrer2013f,Motamedi2022a}. That null hypothesis address a key question: if we preserve the frequency of the word orders, namely we preserve the magnitude of those constraints as reflected in the distribution of word order frequencies, can we still observe a preference for swap distance minimization? 

In the following sections, the principles above are investigated in two scenarios: across dominant or preferred orders in language and also, as a pilot study, within a few languages so as to pave the way for a future large-scale analysis on large ensembles of languages.
Section \ref{sec:materials} presents the syntactic structures and the datasets that will be employed to investigate the effect of entropy minimization and swap distance minimization in syntactic structures with $n=3$ or $n=4$ with $H$ and $\averageswapdistance$. 
Section \ref{sec:methods} presents the null hypotheses and models that will be used to test for the manifestation of entropy minimization and swap distance minimization in languages.
Section \ref{sec:results} shows statistically significant effects of block entropy minimization and swap distance minimization and reveals swap distance minimization effects even when the entropy of orders remains constant, namely swap distance minimization captures constraints on word order that escape entropy minimization. 
Section \ref{sec:discussion} discusses the strength of entropy minimization and swap distance minimization effects and the depth of swap distance minimization.  

\subsection{A word of caution}

$H$ and $\averageswapdistance$ are ways of measuring the skewness of a distribution. It is already well-known that word order distributions are skewed \citep{Cysouw2010a, Ferrer2024b}. The key contribution of this article is to demonstrate that languages prefer distributions that, in addition to being skewed, concentrate probability on word orders that are close in the permutohedron as predicted by the principle of swap distance minimization. The main objective of the article is not descriptive, namely investigating the skewness of the word order distributions but rather theoretical, testing the hypothesis that the principles of entropy minimization and swap distance minimization shape languages and understanding the extent to which the latter surpasses the former. Description comes as a side-effect of our analysis. Furthermore, these forces must be understood as not acting in isolation. Thus one cannot expect {\em a priori} that languages minimize, in the sense of reaching the mathematical minimum of the corresponding statistics, e.g., giving for instance $H=0$, the minimum value of $H$, that is achieved when only one order is used. There are competing constraints and mechanisms that are relevant but are omitted for the sake of simplicity in our analyses (see for instance the discussion on the conflict between entropy minimization and mutual information maximization in word order frequencies \citep{Bentz2017a} or between entropy minimization and dependency distance minimization \citep{Ferrer2013f}). Therefore, when later on we claim that we have found  evidence of entropy minimization or swap distance minimization in languages, we mean that we have found a value of the corresponding statistic that is significantly small, which suffices as evidence of the action of these principles. Thus, such evidence does not imply that the corresponding score has reached its theoretical minimum in a mathematical sense.

\section{Materials}

\label{sec:materials}

The syntactic structures we consider are 
\begin{itemize}
\item 
SOV, formed by subject (S), direct object (O) and verb (V).
\item
VOX, formed by verb (V), direct object (O) and oblique (X).
\item
OVI, formed by direct object (O), verb (V) and indirect object (I).
\item 
SOVI, formed by subject (S), direct object (O), verb (V) and an indirect object (I).
\item
nAND, that is noun phrases formed by a noun (n), an adjective (A), a numeral (N) and a demonstrative (D).
\end{itemize}
Figure \ref{fig:permutohedra_of_linguistic_structures} shows permutohedra for SOVI and nAND structures.

Given a certain syntactic structure, languages reflect a subset of possible orders. For instance, the structure SOV has six possible orders (Figure \ref{fig:word_order_permutation_ring}) while SOVI and nAND have 24 possible orders.
Among the possible orders, there is often one that is dominant or preferred. For instance, concerning SOV structures, SVO order is the dominant order for Mandarin Chinese while SOV order is the preferred order for Urdu \footnote{\url{https://wals.info/languoid/lect/wals_code_urd}}, Hindi \citep{McGregor1977a} and Malayalam \citep{Asher_and_Kumari_1997a}.

We obtain the frequency of a word order from two kinds of sources: the frequency of a dominant or preferred order in a collection of languages or the frequency of each order in a corpus of a specific language. 


A summary of the features of the whole dataset is replicated in Tables \ref{tab:statistical_summary_entropy} and \ref{tab:statistical_summary_average_swap_distance}: the database, the kind of frequency (if the frequency is a corpus frequency then the language is also indicated), the sequence length $n$ of the structure, the structure, the unit of measurement of frequency and a series of basic statistics: $F$, the total frequency, $m$, the number of non-zero probability orders, $\dominanceindex$, the dominance index, $H$, the entropy 
(only in Table \ref{tab:statistical_summary_entropy}) and $\averageswapdistance$, the average swap distance (only in Table \ref{tab:statistical_summary_average_swap_distance}). $F$ indicates the dataset size. When the kind of frequency is a dominant or preferred order, $F$ indicates the number of languages, genera or families; when the kind of frequency is a corpus frequency, $F$ indicates the number of syntactic structures in the corpus. Further details about each dataset are presented below.

\begin{figure}
\caption{\label{fig:permutohedra_of_linguistic_structures} The permutohedra of order $4$ that result from the SOVI structure (top) and the nAND structure (bottom). 
Vertex or edge colors reflect their weight.
For a vertex $i$, the weight is $p_i$, the relative frequency of the word order. If $p_i = 0$ then the color is white. If $p_i > 0$ then the intensity of blue reflects $p_i$. The weight of an edge joining vertices $i$ and $j$ is $p_i p_j$. 
For the SOVI structure, the frequency of each order in a corpus of Hindi-Urdu is borrowed from Table 1 of \citep{Leela2016a} excluding CP constituents.
For the nAND structure, the frequency of each dominant order is its adjusted number of languages according to \citet{Dryer2018a}.  
}
\centering
\includegraphics[width = 0.6 \textwidth]{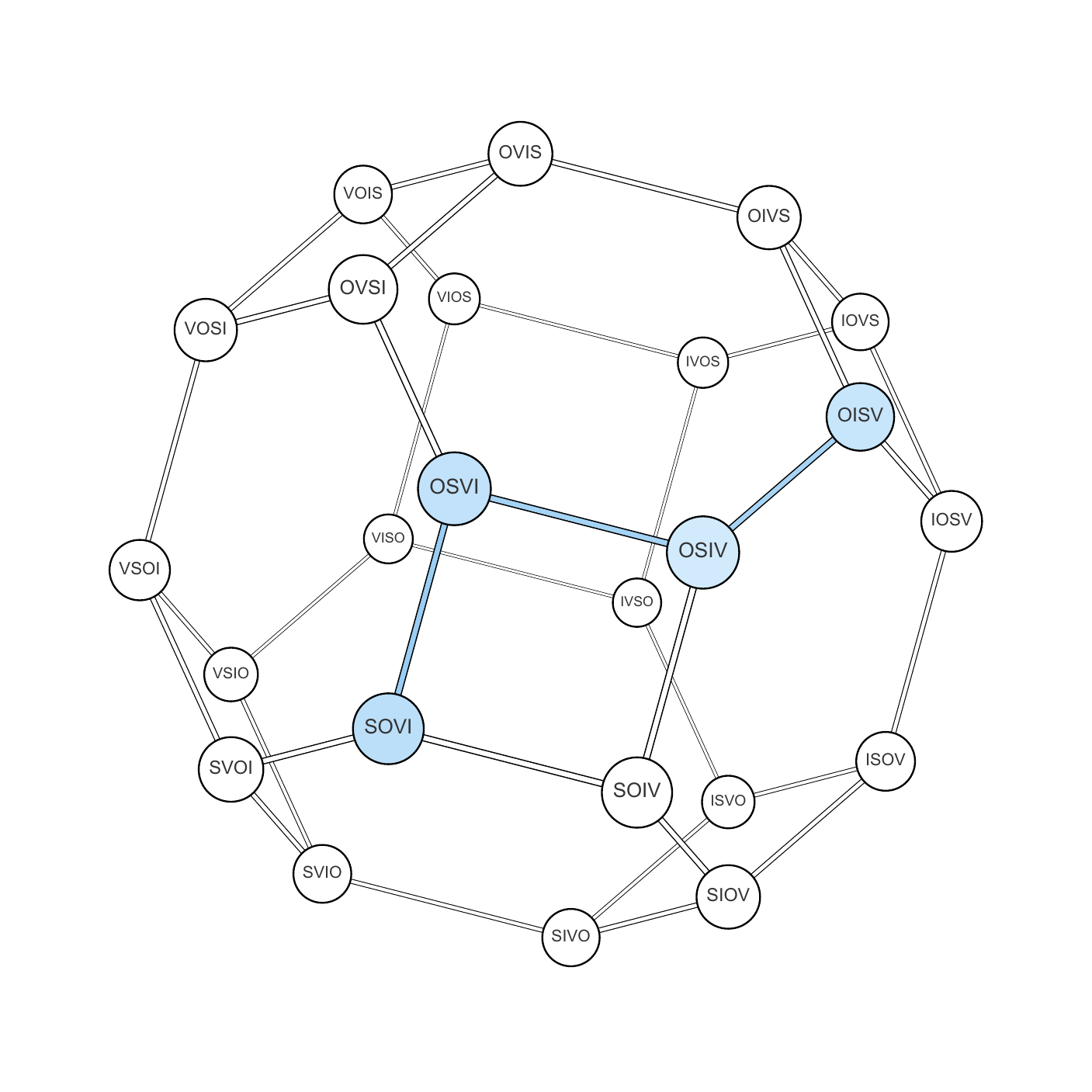}\\
\includegraphics[width = 0.6 \textwidth]{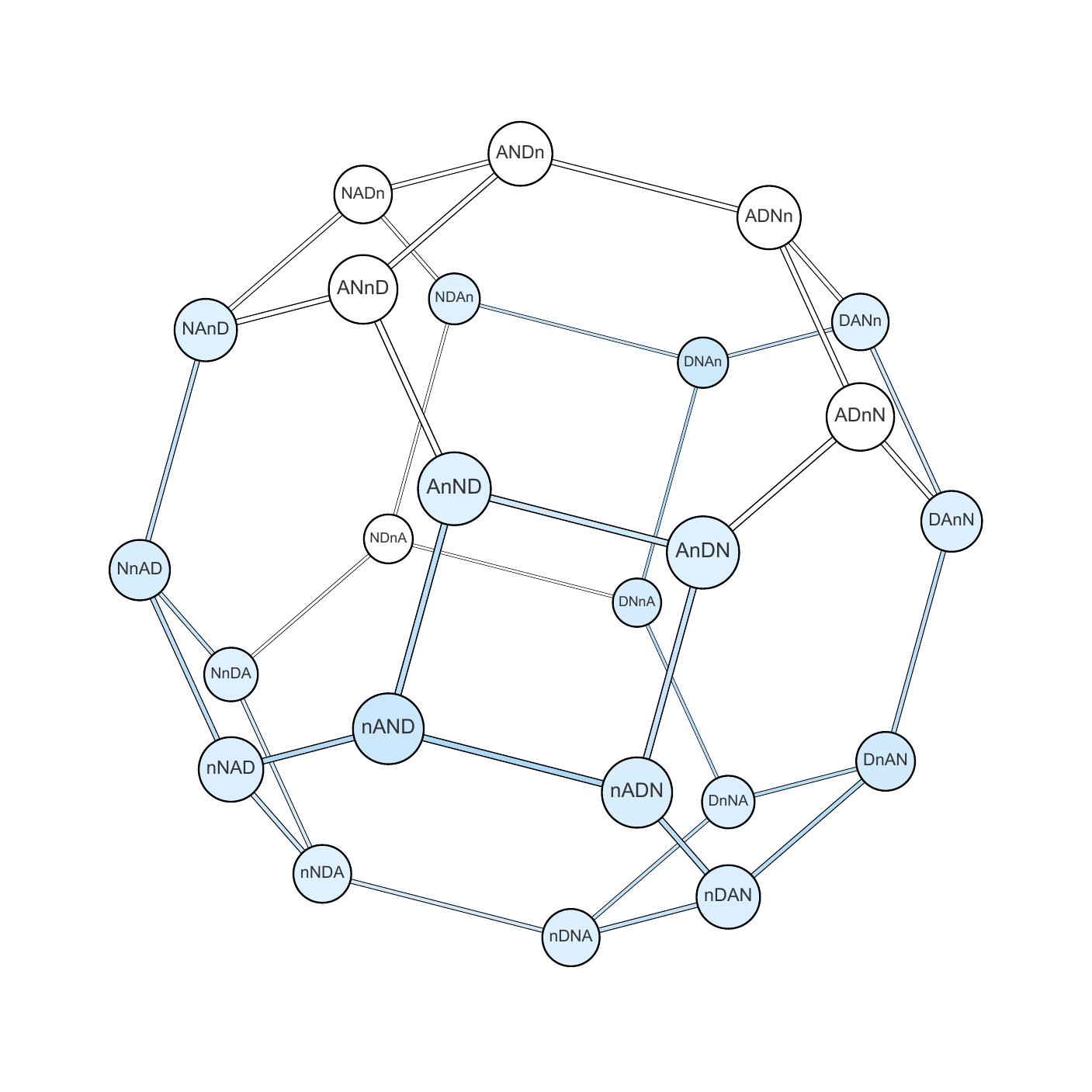}
\end{figure}

\subsection{Frequency of dominant word orders}

The frequency of the dominant order, measured in languages, of SOV and VOX according to WALS are borrowed from \citet{wals-81} and \citet{wals-84}, respectively. Word order frequencies, measured in languages and in families, are borrowed from  \citet{Hammarstroem2016a}. The frequency of the preferred order of nAND, measured in three units: languages, genera and adjusted number of languages, are borrowed from  \citet{Dryer2018a}. The latter unit is a correction on the number of languages to control for genealogy and geography (see \citet{Dryer2018a} for further details on these counts).\footnote{Adjusted number of languages is called adjusted frequency by \citet{Dryer2018a} but we borrow the name from \citet{Martin2019a} because it makes explicit the unit of measurement.}

\subsection{Frequency of word order in corpora}

We obtain the corpus frequency of the orders of SOV, OVI and SOVI from research on word order by Leela and collaborators \citep{Leela2016a,Gavarro2015a,Gavarro2016a}.
The frequency of the orders of SOV, OVI and SOVI in Hindi-Urdu spontaneous speech was obtained from Table 1 of \citep{Leela2016a}. 
In that table, one finds VSO (with frequency 50) and VSOc (with frequency 17). The subindex c indicates that the object O is a subordinate phrase \footnote{Anna Gavarró, personal communication (19 February 2024); Maya Leela, personal communication (14 April 2024).} that is known as Complementizer Phrase (CP) in the terminology of generative grammar \citep{Gavarro2015a}.
We offer two analyses: one where they are merged as VSO (with frequency 67) and another where VSOc is discarded. The one where they are merged is marked with the label CP. 
In Table 1 of \citep{Leela2016a}, one also finds ISVOc (with frequency 9)\footnote{Indeed, it appears with the notation DI S V DO where DI is I in our notation and DO is O in our notation.}. We offer two analyses: one where it was replaced by ISVO (with frequency 9) and another where it was removed. 
The corpus frequency of the orders of SOV in Hindi-Urdu was obtained also from child directed speech (CDS), i.e. adult speech directed to children (Table 2 of \citep{Leela2016a}). This dataset is marked as CDS.
The corpus frequency of the order of SOV in Malayalam was obtained from Table 3 of \citep{Leela2016a}.
In Table 1 of \citep{Leela2016a}, the order IOV appears twice, one with frequency 42 and the other with frequency 12. Those numbers were summed to produce 52 as the frequency of OIV. \footnote{According to Maya Leela (personal communication, 16 April 2024), there was a typo in the original table and she told us to sum both numbers to fix the problem.}
Finally, we also include the absolute frequencies of the orders of the SOV structure from Lev-Ari's word order experiments \citep{Lev-Ari2023a}. 

\section{Methods}

\label{sec:methods}
 
For a certain structure in a dataset, we compute $F$ (Equation \ref{eq:total_frequency}), $H$, the entropy of the word orders, and $\averageswapdistance$, the average swap distance between the possible orders of that structure.
$H$ is computed using the plug-in estimator (Appendix \ref{app:subsec_entropy}).
To reduce numerical error, we compute $H$ via the expanded equivalent expression (Equation \ref{eq:entropy_plug-in_estimator_expanded})
and $\averageswapdistance$ via
\begin{eqnarray*}
\averageswapdistance = \frac{1}{F^2} \sum_{i=1}^{N} f_i \sum_{j=1}^{N} f_j d_{ij},
\end{eqnarray*}
where $d_{ij}$ is the swap distance between orders $i$ and $j$, i.e. the distance in edges between $i$ and $j$ in the permutohedron.

\subsection{Null hypotheses}

We determine if $H$ is significantly small as expected by the entropy minimization principle or if $\averageswapdistance$ is significantly small as expected by the swap distance minimization principle using the following null hypotheses, which are sorted in order of increasing (non-decreasing) strength:
\begin{enumerate}
    \item 
    A die rolling experiment ($dr$). The experiment consists of rolling a fair die $F$ times. The die has $n!$ sides, each corresponding to one of the orders of the structure. The frequency of each word order is the frequency of the corresponding side over the $F$ rolls.
    \item 
    A random walk. Orders are chosen by performing an unbiased random walk on the permutohedron graph. Each random walk starts on some vertex of that graph and lasts for an arbitrary number of steps.       
    \item 
    A Polya urn experiment ($Pu$) with balls of $n!$ distinct colors. Each color corresponds to a word order. The experiment consists of initializing the urn with $n!$ balls of distinct colors and then choosing a ball with replacement and adding a copy of it to the urn until the urn contains $F$ balls.    
    \item 
    A random permutation ($rp$), that consists of shuffling the $n!$ word order frequencies. That null hypothesis preserves the empirical distribution of word order frequencies and thus preserves $H$.   
\end{enumerate} 
Notice that all the null hypotheses except the random walk neglect the structure of the permutohedron in the sense that word order frequencies are produced ignoring the permutohedron structure. However, as the permutohedron is a regular graph (all vertices have the same degree), it is easy to see that the probability of visiting a certain vertex is the same for all vertices regardless of the initial vertex or the length of the walk \citep{Lovasz1993a}. Therefore, the random walk null hypothesis is equivalent to the die rolling null hypothesis.

Null hypotheses can be classified into preserving (those that preserve the true actual distribution of word orders in the data) and non-preserving (those that do not preserve that distribution). The die rolling experiment, the random walk and the Polya urn experiment are non-preserving. The random permutation is preserving. 
The Polya urn was chosen for it simplicity and its ability to reproduce skewed distributions, as it is rather easy to reject the die rolling null hypothesis given the skewed nature of word order distributions \citep{Cysouw2010a,Ferrer2024b}.
However, there is an inherent arbitrariness in the choice of the non-preserving null models. They are a tiny sample of the ensemble of null models that would produce a higher or a lower test statistic ($H$ or $\averageswapdistance$) than the empirical data. Crucially, we are considering a random permutation, a preserving null hypothesis that cuts arbitrariness for tests on $\averageswapdistance$ from the ensemble of non-preserving distribution by addressing the following question: if we use the best model for the data, namely, the data itself, can we still find evidence of swap distance minimization? Some arbitrariness for tests on $H$ remains and is critical for the interpretation of the results: the distinct non-preserving models tell us about the distinct levels of skewness in the data (low skewness for die rolling, higher skewness for the Polya urn) that in turn inform about the intensity of the principle under investigation. Covering more levels of skewness or modeling the precise distribution of frequencies is beyond the scope of the present article.

In Table \ref{tab:statistical_summary_average_swap_distance}, the expected value of $\averageswapdistance$ for the die rolling null hypothesis, $\averageswapdistance_{dr}$, is calculated with Equation \ref{eq:expected_average_swap_distance_die_rolling_main_text}; the expected value of $\averageswapdistance$ for the random permutation null hypothesis, $\averageswapdistance_{rp}$, is calculated with Equation \ref{eq:expected_average_swap_distance_random_permutation_main_text}. 
The expected value of $\averageswapdistance$ for the Polya urn null hypothesis, $\averageswapdistance_{Pu}$, is estimated by means of a Monte Carlo procedure that consists of running $T$ Polya urn experiments. 
$\averageswapdistance_{max}$ is calculated by means of Equation \ref{eq:average_swap_distance_max_main_text}.

The frequencies of the orders of nAND structures are not integer numbers when measured in adjusted number of languages \citep{Dryer2018a}. In that case, $F$ is rounded to the nearest integer for the die rolling and the Polya urn experiments. 

\subsection{Hypothesis testing}

Given a null hypothesis and a score ($H$ or $\averageswapdistance$), we run $T$ experiments and estimate $\mathbb{P}$, the left $p$-value, as the proportion of experiments such that $x' \leq x $, where $x'$ is the value of $\averageswapdistance$ in a experiment. If the $p$-value estimated by the Monte Carlo procedure is 0, the estimate is replaced by $1/T$, a likely upper bound of the actual $p$-value. For the permutation null hypothesis and $n=3$, $\mathbb{P}$ is computed exactly by generating all permutations of the vector $p$.  

To control for multiple comparisons, we apply a Holm correction to $p$-values obtained with the same null hypothesis. We use a $<$ sign to indicate $p$-values that are a likely upper bound before or after applying the Holm correction.   

As the manifestation of the principles (entropy minimization or swap distance minimization) may be weak, we also consider a test that, given a score $x$ ($H$ or $\averageswapdistance$) and some null hypothesis, compares $x$, the value of the score against $x'$, the value of the score under the null hypothesis across all cases by means of a Wilcoxon signed-rank test \citep{Conover1999a}. This test is often used to check if a certain medical treatment has a significant effect on patients. The test is paired in the sense that one has two paired (or matched) samples; a measurement on a patient before the treatment in one sample is paired with a measurement on the same patient after the treatment in the other sample. Here we are comparing the original value of the score against its paired value after randomization by some null hypothesis that preserves just some information about the original configuration (e.g., all the null hypotheses just preserve $F$, excluding the random permutation hypothesis, that preserves the multiset of word frequencies). If the original values came from the null hypothesis, no significant difference between $x$ and its randomized counterpart $x'$ should be found. Here we used a left-sided Wilcoxon signed-rank test as we assume that, if the principle is acting, a tendency towards $x < x'$ is expected. Put differently, the test checks if the randomization of the null hypothesis has some effect towards larger values. If it does, that implies that the corresponding optimization principle (entropy minimization or swap distance minimization) has some effect.  

In this article, we assume a significance level of 0.05 just as an orientation for discussion. 

\subsection{The statistical power of the random permutation test}

The statistical power of the permutation test, i.e. its {\em a priori} capacity to reject the null hypothesis, is limited. 
$\mathbb{P}$, the left $p$-value of the permutation test, is predetermined to be ``large'' in certain conditions. When $n = 3$, $\mathbb{P} \geq \frac{1}{60}$ and the lower bound of $\mathbb{P}$ can be even greater depending on $m$, the number of non-zero probability orders, or the location of the non-zero orders in the permutohedron. An accurate lower bound of $\mathbb{P}$ tends to increase as $m$ increases. A detailed mathematical analysis is presented in Appendix \ref{subsec:p_value_permutation_test}. The take home message is that the permutation test on 
$\averageswapdistance$ may fail to reject the null hypothesis because of lack of statistical power, not because the hypothesis of swap distance minimization is incorrect.

\section{Results}

\label{sec:results}

$m$, the number of orders that have non-zero frequency, is maximum ($N=m$) in most cases when $n=3$ (Table \ref{tab:statistical_summary_entropy}). When $n=4$, there is a big gap between $m$ and $N$ specially for the SOIV structure (Table \ref{tab:statistical_summary_entropy}), suggesting a certain difficulty of languages to cover the whole space of possible permutations. The next subsections shed light on the possible nature of such difficulty. 

\subsection{Entropy}

Table \ref{tab:statistical_summary_entropy} shows the actual entropy ($H$), its expected value in a die-rolling experiment ($H_{dr}$) and in a Polya urn experiment ($H_{Pu}$) for all syntactic structures. 
In all cases, $H < H_{Pu} < H_{dr}$. 
The fact $H_{dr}$ is always the largest is not very surprising as one expects that, $H_{dr} \approx H_{max}$, where $H_{max}=\log N$ is the theoretical maximum $H$.
Then, it is not surprising that $H$ is always significantly small with respect to die rolling. Under the Polya urn null hypothesis, $H$ is always significantly small when $n=4$ and never when $n=3$ (Table \ref{tab:statistical_summary_entropy}). 
It is not very surprising to find entropy minimization effects in all cases when $n=4$ as the space of possible permutations is wider. 
When $n=3$, $H < H_{Pu}$ in all 10 cases. How likely is it that this has happened by chance?
A one-sided Wilcoxon signed-rank test supports a tendency for 
$H < H_{Pu}$
as expected by entropy minimization ($V = 0$, $p$-value $= 9.8 \cdot 10^{-4}$).

\begin{landscape}

\begin{table}
\caption{\label{tab:statistical_summary_entropy} Summary of the statistical information by database, kind, the sequence length $n$, the structure (struct.) and unit of measurement of frequency: $F$, the total frequency, $m$, the number of non-zero probability orders, $\dominanceindex$, the dominance index, $H$, the entropy, 
$H_{Pu}$, the expected value of $H$ in a Polya urn experiment,
$H_{dr}$, the expected value of $H$ in a die rolling experiment, $H_{max} = \log(n!)$, the maximum value of $H$, and the $p$-values ($\mathbb{P}$) of left-sided tests on $H$: Polya urn ($Pu$) and die rolling ($dr$). Each column of $p$-values has been adjusted with a Holm correction; the original $p$-value is shown below the adjusted $p$-value. The unit ``adj. langs.'' stands for adjusted number of languages.
}
\footnotesize
\centering
\begin{tabular}{@{} llllllllllllll @{}}
\hline
Database & Kind & $n$ & Struct. & Unit & $F$ & $m$ & $\dominanceindex$ & $H$ & 
$H_{Pu}$ & $H_{dr}$ & $H_{max}$ & 
$\mathbb{P}_{Pu}$ & $\mathbb{P}_{dr}$ \\
\hline
\input{tables/entropy_table}\\
\hline
\end{tabular}
\end{table}

\end{landscape}

\subsection{Average swap distance}

Table \ref{tab:statistical_summary_average_swap_distance} shows the actual value of $\averageswapdistance$ and its expected value in a die-rolling experiment ($\averageswapdistance_{dr}$), a Polya urn experiment ($\averageswapdistance_{Pu}$) and a random permutation ($\averageswapdistance_{rp}$) for all syntactic structures. 
In all cases, $\averageswapdistance < \averageswapdistance_{Pu} < \averageswapdistance_{dr}$. The finding that $\averageswapdistance < \averageswapdistance_{dr}$ is not very surprising because the expected values for a die rolling experiment are close to $\averageswapdistance_{max}$, the theoretical maximum $\averageswapdistance$ ($\averageswapdistance_{max} = 1.5$ when $n=3$ and $\averageswapdistance_{max}=3$ when $n=4$).

When it comes to significance, we find that $\averageswapdistance$ is always significantly small with respect to die rolling, never significant with respect to a random permutation and only significantly small with respect to a Polya urn when the structure is SOVI with and without CP in Hindi-Urdu and borderline significant in the SOV structure without CP in Hindi-Urdu (Table \ref{tab:statistical_summary_average_swap_distance}). 
However, notice that, for the SOV structure in Malayalam, the raw $p$-value of the permutation test (the $p$-value before Holm's correction), is $1/30 \approx 0.03$, which matches the minimum $p$-value for $n=3$ and $m=4$ 
(Appendix \ref{subsec:p_value_permutation_test}; Table \ref{tab:same_average_swap_distance}).

As for the Polya urn null hypothesis, it is not very surprising to find that swap distance minimization effects more easily when $n=4$ (SOVI structure) as the space of possible permutations is bigger (Figure \ref{fig:permutohedra_of_linguistic_structures}). 
When $n=3$, the statistical tests never find swap distance minimization effects but $\averageswapdistance < \averageswapdistance_{Pu}$ in 10 out of 10 cases. How likely is it that this has happened by chance?
A one-sided Wilcoxon signed-rank test supports a tendency for $\averageswapdistance < \averageswapdistance_{Pu}$ as expected by swap distance minimization ($V = 0$, $p$-value $= 9.8 \cdot 10^{-4}$).

Regarding the random permutation null hypothesis, $\averageswapdistance$ was never significantly low after Holm's correction but before it was in a few cases: the SOV structure including CP in Hindi-Urdu, the SOV structure in Malayalam and SOVI structures with and without CP in Hindi-Urdu (Figure \ref{fig:permutohedra_of_linguistic_structures}).
The finding of swap distance minimization for the SOV structure in Malayalam confirms previous findings on that corpus using a correlation statistic that assumes that SOV is the source order \citep{Ferrer2023a} while $\averageswapdistance$ does not make any assumption about the existence of a primary source order or the number of sources. 
Notice also that $\averageswapdistance < \averageswapdistance_{rp}$ in 11 out 15 cases (the failures were the VOX structure in WALS and the nAND structure independently of the unit of measurement of frequency). A one-sided Wilcoxon signed-rank test supports a tendency for $\averageswapdistance < \averageswapdistance_{rp}$ as expected by swap distance minimization ($V = 17$, $p$-value $= 6.2 \cdot 10^{-3}$).

Table \ref{tab:Wilcoxon_signed_rank_test_summary} shows the outcome of the Wilcoxon signed-rank test over all conditions, including conditions that we have not selected above for the sake of completeness and to shed some light on the robustness of the tests beyond the critical cases discussed in the main text.

\begin{landscape}

\begin{table}
\caption{\label{tab:statistical_summary_average_swap_distance} Summary of the statistical information by database, kind, 
the sequence length $n$, structure (struct.) and unit of measurement of frequency: $F$, the total frequency, $m$, the number of non-zero probability orders, $\averageswapdistance$, the average swap distance, $\averageswapdistance_{rp}$, the expected value of $\averageswapdistance$ in a random permutation,
$\averageswapdistance_{Pu}$, the expected value of $\averageswapdistance$ in a Polya urn experiment,
$\averageswapdistance_{dr}$, the expected value of $\averageswapdistance$ in a die rolling experiment, $\averageswapdistance_{max}$, the maximum value of $\averageswapdistance$ in a permutohedron of order $n$, and the left $p$-values ($\mathbb{P}$) of tests on $\averageswapdistance$: random permutation ($rp$), Polya urn ($Pu$) and die rolling ($dr$). Each column of $p$-values has been adjusted with a Holm correction; the original $p$-value is shown below the adjusted $p$-value. The unit ``adj. langs.'' stands for adjusted number of languages.
}
\footnotesize
\centering
\begin{tabular}{@{} lllllllllllllll @{}}
\hline
Database & Kind & $n$ & Struct. & Unit & $F$ & $m$ & $\averageswapdistance$ & 
$\averageswapdistance_{rp}$ & $\averageswapdistance_{Pu}$ & $\averageswapdistance_{dr}$ & 
$\averageswapdistance_{max}$ & $\mathbb{P}_{rp}$ & $\mathbb{P}_{Pu}$ & $\mathbb{P}_{dr}$ \\
\hline
\input tables/average_swap_distance_table
\hline
\end{tabular}
\end{table}

\end{landscape}

\begin{table}
\caption{\label{tab:Wilcoxon_signed_rank_test_summary}  Summary of the Wilcoxon signed-rank test for each inequality over the whole dataset (``all'') and for distinct subgroups determined by structure size ($n$) or the nature of the word order frequencies (``dominant'' for word order preferences across languages and ``corpus'' for word order frequencies for the corpus of a specific language). $V$ is the value of the statistic of the Wilcoxon signed-rank test and $\mathbb{p}$ is the corresponding $p$-value. Given an inequality, we do not control for multiple comparisons (with a Holm correction) as the aim of this table is not to fish for significance across conditions; this table is essentially provided for the sake of completeness. 
}
\footnotesize
\centering
\begin{tabular}{@{} llllllllllllllll @{}}
\hline
& \multicolumn{2}{c}{all} & \multicolumn{2}{c}{$n=3$} & \multicolumn{2}{c}{$n=4$} & \multicolumn{2}{c}{dominant} & \multicolumn{2}{c}{corpus} \\ 
\cmidrule(lr){2-3} \cmidrule(lr){4-5} \cmidrule(lr){6-7} \cmidrule(lr){8-9} \cmidrule(lr){10-11}
& $V$ & $\mathbb{P}$ & $V$ & $\mathbb{P}$ & $V$ & $\mathbb{P}$ & $V$ & $\mathbb{P}$ & $V$ & $\mathbb{P}$ \\
\hline
\input tables/Wilcoxon_signed_rank_test_table
\hline
\end{tabular}
\end{table}

\section{Discussion}

\label{sec:discussion}

\subsection{Theory}

We have presented entropy minimization and swap distance minimization as separate principles but they are related via the scores that we have used for each. $H$, $\dominanceindex$, and $\averageswapdistance$ are diversity scores: they hit zero, their minimum value when diversity is minimum (only one word is used) and hit their maximum value when orders are equally likely (this is true for $\averageswapdistance$ at least for $n\leq \nmax$). However, $\averageswapdistance$ hits is maximum value in conditions where $H$ is not maximum (Property \ref{prop:average_swap_distance equally_likely_orders_and_more}).
$H$ and $\dominanceindex$ care only about the probabilities of each order while $\averageswapdistance$ takes also into account the structure of the permutothedron. The relationship between entropy minimization and swap distance minimization can be understood via $\dominanceindex$, that is another diversity score that only cares about word order probabilities as $H$ does. $\dominanceindex$ appears in various properties of $\averageswapdistance$: involved in an upper bound for $\averageswapdistance$ (Property \ref{prop:upper_bound_Simpson_index}) or in the expected value of $\averageswapdistance$ (Property \ref{prop:expected_average_swap_distance_random_permutation}).
Indeed, both $H$ and $\averageswapdistance$ are related via the Rényi entropy, that is defined as \citep{Renyi1961a}
\begin{eqnarray*}
R_\alpha = \frac{1}{1 - \alpha} \log\left( \sum_{i=1}^n p_i^\alpha \right).
\end{eqnarray*}
$H$ and $S$ (or $\dominanceindex$) are particular cases of Rényi entropy
because
\begin{eqnarray*}
\lim_{\alpha \rightarrow 1} R_{\alpha}= H\\
R_2 = - \log(S).
\end{eqnarray*}
Very similar arguments can be made with Tsallis entropy \citep{Tsallis1988a}. Therefore, entropy minimization in word order could be an implication of the more fundamental principle of swap distance minimization. Alternatively, the principle of entropy minimization and the principle of swap distance minimization may be implications of the minimization of some generalized entropy that is aware of the permutohedron structure. These possibilities should be the subject of future research. 

\subsection{Experiments}

Here we have investigated the manifestation of word order entropy minimization and swap distance minimization in word order applying a novel score for swap distance minimization.

\subsubsection{Entropy minimization}

We have found very strong support for entropy minimization as expected by generalizing the principle of word entropy minimization \citep{Ferrer2015b,Ferrer2013f} with respect to the die rolling null hypothesis and also, when $n=4$ with respect to the Polya urn null hypothesis. When $n=3$, there is still a significant signal of entropy minimization under the Polya urn null hypothesis that is unveiled by a global analysis over all structures with $n=3$. 

\subsubsection{Swap distance minimization}

With respect to the null hypothesis of a die rolling experiment, we have found super strong support for swap distance minimization but this is a rather trivial result because the word orders of a certain syntactic structure are not equally likely \citep{Cysouw2010a} and we have already provided support for entropy minimization in word order variation. Interestingly, the Polya urn is able to produce a more skewed distribution of word orders but its dynamics is not driven explicitly by any cost minimization principle. 
Thus, the Polya urn serves as baseline for the intensity of swap distance minimization effects.
According to the Polya urn null hypothesis, we have found a weak signal of the principle in triplets but a strong signal for the SOIV structure. 
In triplets, the actual $\averageswapdistance$ was significantly small with respect to a Polya urn experiment only in the SOV structure in Hindi-Urdu. Such a signal is weak but significant because of a tendency for $\averageswapdistance < \averageswapdistance_{Pu}$ when $n=3$.
In quadruplets, the actual $\averageswapdistance$ was significantly small compared to a Polya experiment for the SOVI structure with and without CP \citep{Dryer2018a}.

The fact that $entropy$ or $\averageswapdistance$ is significantly small with respect to die rolling or Polya urn does not imply that word orders are constrained to have low $H$ or $\averageswapdistance$. It could simply mean that actual word order frequencies are more skewed than expected by these stochastic processes. As researchers have tried to reduce the explanation of the frequency distribution of linguistic units in languages to simplistic stochastic processes since G. K. Zipf's times \citep{Miller1957,Rapoport1982,Suzuki2004a}, we invite skeptical researchers about the reality of swap distance  minimization to use any mechanism they wish for the origins of the frequency distribution. Here we have taken the best model for the $N$ word order probabilities, i.e. the $N$ relative frequencies obtained from the dataset, to build the random permutation null hypothesis and to demonstrate that swap distance minimization effects cannot be reduced to the distribution of word order frequencies. 

From a cognitive or linguistic standpoint, the random permutation null hypothesis is able to control for other word order principles or mechanisms. We do not need to know exactly which. In particular, the random permutation null hypothesis addresses the following question: if we accept that word order frequencies are determined by other word order principles (principles other than swap distance minimization) and accordingly preserve the frequency distribution, can we still see the action of swap distance minimization? The finding is that there is a weak but significant sign of the action of swap distance minimization across conditions. 
Therefore, we conclude that, even when the entropy of word orders is kept constant as in the random permutation null hypothesis, we still detect the effect of swap distance minimization. Therefore, neither the general principle of entropy minimization \citep{Ferrer2015b,Ferrer2013f} nor alternatives to swap distance minimization can fully explain word order preferences. 

It could still be argued that swap distance minimization is a rather weak principle compared to entropy minimization. Indeed, the $p$-values produced by the permutation test (Table \ref{tab:statistical_summary_average_swap_distance}) were very high for $n=3$. However, this may be caused by a lack of statistical power of the test for $n=3$. The minimum $p$-value is $1/60$ because of the large number of symmetries, and this lower bound increases as $m$ reduces from $m=5$ down to $1$ (Appendix \ref{subsec:p_value_permutation_test}). Thus, symmetries and the difficulty of languages to cover the whole permutation space may be shadowing the action of swap distance minimization in languages. The paradox is that, in case swap distance was minimized, i.e. $\averageswapdistance = 0$, the left $p$-value of the permutation test would be $\mathbb{P}=1$ independently of the value of $n$. Thus the permutation test cannot detect swap distance minimization if $m$ is small (Figure \ref{fig:p_value_lower_bound}, Property \ref{prop:statistical_power}), although that small number can be interpreted as a consequence of swap distance minimization. Even when $n = 4$, we failed to find evidence of swap distance minimization for the nAND structure (Table \ref{tab:statistical_summary_average_swap_distance}). The fact that 
$\averageswapdistance > \averageswapdistance_{rp}$ in nAND structures suggests that swap distance minimization is not acting in these structures. However, notice that the two most frequently preferred orders, nAND and its mirror, DNAn are located at maximum distance in the permutohedron (Figure \ref{fig:permutohedra_of_linguistic_structures}). We hypothesize that the principle is still acting but that $\averageswapdistance$ is not well-suited to capture its effects when variants are produced from distant sources. $\averageswapdistance$ is maximized when only two orders have non-zero probability and they are located at maximum distance (Property \ref{prop:average_swap_distance equally_likely_orders_and_more}). For the nAND structure, $\averageswapdistance > 2.8$ while, theoretically, $\averageswapdistance \leq 3$ (Equation \ref{eq:average_swap_distance_max_main_text}; Property \ref{prop:global_upper_bound}). That hypothesis should be the subject of future research. 

\subsection{Concluding remarks}

We have just investigated some syntactic structures and, within languages, we have investigated the effect of these principles just for a few languages. Our article was intended to set the stage for a future massive exploration of languages that does justice to the linguistic diversity of our planet.

\iftoggle{anonymous}{}{

\section*{Acknowledgments}

We are grateful to Shiri Lev-Ari for helpful discussions on swap distance minimization and assistance on her word frequency counts \citep{Lev-Ari2023a}. We are also grateful to Maya Leela and Anna Gavarró for advice on Leela's corpora \citep{Leela2016a}. 

This research was presented during an IFISC Seminar in Palma de Mallorca (Spain) on 8 March 2023. The third author thanks the organizers and participants of that event for valuable questions and discussions. This research is supported by the grant PID2024-155946NB-I00 funded by Ministerio de Ciencia, Innovación y Universidades (MICIU), Agencia Estatal de Investigación (AEI/10.13039/501100011033) and the European Social Fund Plus (ESF+).
This research is also supported by a recognition 2021SGR-Cat (01266 LQMC) from AGAUR (Generalitat de Catalunya).
}

\bibliographystyle{apacite}

\bibliography{all}

\appendix

\input{appendix}

\end{document}

%% file: tables/entropy_table.tex
WALS & dominant order & 3 & SOV & langs. & 1187 & 6 & 0.6 & 1.06 & 1.45 & 1.79 & 1.79 & $0.21$ & $<1.5\cdot 10^{-6}$\\
 & & & & & & & & & & & & $0.03$ & $<1\cdot 10^{-7}$ \\
WALS & dominant order & 3 & VOX & langs. & 333 & 5 & 0.56 & 1.09 & 1.46 & 1.78 & 1.79 & $0.21$ & $<1.5\cdot 10^{-6}$\\
 & & & & & & & & & & & & $0.03$ & $<1\cdot 10^{-7}$ \\
Hammarstr\"{o}m & dominant order & 3 & SOV & langs. & 5128 & 6 & 0.62 & 1.13 & 1.45 & 1.79 & 1.79 & $0.21$ & $<1.5\cdot 10^{-6}$\\
 & & & & & & & & & & & & $0.05$ & $<1\cdot 10^{-7}$ \\
Hammarstr\"{o}m & dominant order & 3 & SOV & families & 340 & 6 & 0.47 & 0.94 & 1.46 & 1.78 & 1.79 & $0.1$ & $<1.5\cdot 10^{-6}$\\
 & & & & & & & & & & & & $9.6\cdot 10^{-3}$ & $<1\cdot 10^{-7}$ \\
Lev-Ari & experiments & 3 & SOV & frequency & 13985 & 6 & 0.58 & 1.13 & 1.45 & 1.79 & 1.79 & $0.21$ & $<1.5\cdot 10^{-6}$\\
 & & & & & & & & & & & & $0.05$ & $<1\cdot 10^{-7}$ \\
Leela & Hindi-Urdu CP & 3 & SOV & frequency & 3206 & 6 & 0.51 & 1.06 & 1.45 & 1.79 & 1.79 & $0.21$ & $<1.5\cdot 10^{-6}$\\
 & & & & & & & & & & & & $0.03$ & $<1\cdot 10^{-7}$ \\
Leela & Hindi-Urdu & 3 & SOV & frequency & 3139 & 5 & 0.48 & 0.98 & 1.45 & 1.79 & 1.79 & $0.14$ & $<1.5\cdot 10^{-6}$\\
 & & & & & & & & & & & & $0.02$ & $<1\cdot 10^{-7}$ \\
Leela & Hindi-Urdu CDS & 3 & SOV & frequency & 398 & 6 & 0.64 & 1.25 & 1.46 & 1.79 & 1.79 & $0.24$ & $<1.5\cdot 10^{-6}$\\
 & & & & & & & & & & & & $0.12$ & $<1\cdot 10^{-7}$ \\
Leela & Hindi-Urdu & 3 & OVI & frequency & 138 & 4 & 0.61 & 1.08 & 1.47 & 1.77 & 1.79 & $0.21$ & $<1.5\cdot 10^{-6}$\\
 & & & & & & & & & & & & $0.03$ & $<1\cdot 10^{-7}$ \\
Leela & Malayalam & 3 & SOV & frequency & 3763 & 4 & 0.74 & 1.37 & 1.45 & 1.79 & 1.79 & $0.28$ & $<1.5\cdot 10^{-6}$\\
 & & & & & & & & & & & & $0.28$ & $<1\cdot 10^{-7}$ \\
Dryer & dominant order & 4 & nAND & langs. & 576 & 18 & 0.83 & 2.13 & 2.8 & 3.16 & 3.18 & $1.3\cdot 10^{-4}$ & $<1.5\cdot 10^{-6}$\\
 & & & & & & & & & & & & $1\cdot 10^{-5}$ & $<1\cdot 10^{-7}$ \\
Dryer & dominant order & 4 & nAND & genera & 322 & 18 & 0.86 & 2.28 & 2.81 & 3.14 & 3.18 & $8.4\cdot 10^{-4}$ & $<1.5\cdot 10^{-6}$\\
 & & & & & & & & & & & & $7\cdot 10^{-5}$ & $<1\cdot 10^{-7}$ \\
Dryer & dominant order & 4 & nAND & adj. langs. & 212 & 18 & 0.88 & 2.42 & 2.83 & 3.12 & 3.18 & $6.2\cdot 10^{-3}$ & $<1.5\cdot 10^{-6}$\\
 & & & & & & & & & & & & $5.6\cdot 10^{-4}$ & $<1\cdot 10^{-7}$ \\
Leela & Hindi-Urdu CP & 4 & SOVI & frequency & 320 & 5 & 0.74 & 1.42 & 2.81 & 3.14 & 3.18 & $<1.5\cdot 10^{-6}$ & $<1.5\cdot 10^{-6}$\\
 & & & & & & & & & & & & $<1\cdot 10^{-7}$ & $<1\cdot 10^{-7}$ \\
Leela & Hindi-Urdu & 4 & SOVI & frequency & 311 & 4 & 0.72 & 1.33 & 2.81 & 3.14 & 3.18 & $<1.5\cdot 10^{-6}$ & $<1.5\cdot 10^{-6}$\\
 & & & & & & & & & & & & $<1\cdot 10^{-7}$ & $<1\cdot 10^{-7}$ \\

%% file: tables/average_swap_distance_table.tex
WALS & dominant order & 3 & SOV & langs. & 1187 & 6 & 0.76 & 1.08 & 1.29 & 1.5 & 1.5 & $1$ & $0.12$ & $<1.5\cdot 10^{-6}$\\
 & & & & & & & & & & & & $0.18$ & $0.01$ & $<1\cdot 10^{-7}$ \\
WALS & dominant order & 3 & VOX & langs. & 333 & 5 & 1.08 & 1 & 1.29 & 1.5 & 1.5 & $1$ & $0.59$ & $<1.5\cdot 10^{-6}$\\
 & & & & & & & & & & & & $0.7$ & $0.12$ & $<1\cdot 10^{-7}$ \\
Hammarstr\"{o}m & dominant order & 3 & SOV & langs. & 5128 & 6 & 0.82 & 1.12 & 1.29 & 1.5 & 1.5 & $1$ & $0.16$ & $<1.5\cdot 10^{-6}$\\
 & & & & & & & & & & & & $0.18$ & $0.02$ & $<1\cdot 10^{-7}$ \\
Hammarstr\"{o}m & dominant order & 3 & SOV & families & 340 & 6 & 0.74 & 0.85 & 1.29 & 1.5 & 1.5 & $1$ & $0.1$ & $<1.5\cdot 10^{-6}$\\
 & & & & & & & & & & & & $0.27$ & $9.2\cdot 10^{-3}$ & $<1\cdot 10^{-7}$ \\
Lev-Ari & experiments & 3 & SOV & frequency & 13985 & 6 & 0.83 & 1.04 & 1.29 & 1.5 & 1.5 & $1$ & $0.16$ & $<1.5\cdot 10^{-6}$\\
 & & & & & & & & & & & & $0.2$ & $0.02$ & $<1\cdot 10^{-7}$ \\
Leela & Hindi-Urdu CP & 3 & SOV & frequency & 3206 & 6 & 0.72 & 0.91 & 1.29 & 1.5 & 1.5 & $0.43$ & $0.1$ & $<1.5\cdot 10^{-6}$\\
 & & & & & & & & & & & & $0.03$ & $8.6\cdot 10^{-3}$ & $<1\cdot 10^{-7}$ \\
Leela & Hindi-Urdu & 3 & SOV & frequency & 3139 & 5 & 0.67 & 0.87 & 1.29 & 1.5 & 1.5 & $0.73$ & $0.07$ & $<1.5\cdot 10^{-6}$\\
 & & & & & & & & & & & & $0.07$ & $5.3\cdot 10^{-3}$ & $<1\cdot 10^{-7}$ \\
Leela & Hindi-Urdu CDS & 3 & SOV & frequency & 398 & 6 & 0.99 & 1.15 & 1.29 & 1.5 & 1.5 & $1$ & $0.39$ & $<1.5\cdot 10^{-6}$\\
 & & & & & & & & & & & & $0.17$ & $0.06$ & $<1\cdot 10^{-7}$ \\
Leela & Hindi-Urdu & 3 & OVI & frequency & 138 & 4 & 0.83 & 1.1 & 1.29 & 1.49 & 1.5 & $1$ & $0.15$ & $<1.5\cdot 10^{-6}$\\
 & & & & & & & & & & & & $0.2$ & $0.02$ & $<1\cdot 10^{-7}$ \\
Leela & Malayalam & 3 & SOV & frequency & 3763 & 4 & 1.16 & 1.34 & 1.29 & 1.5 & 1.5 & $0.43$ & $0.59$ & $<1.5\cdot 10^{-6}$\\
 & & & & & & & & & & & & $0.03$ & $0.2$ & $<1\cdot 10^{-7}$ \\
Dryer & dominant order & 4 & nAND & langs. & 576 & 18 & 2.84 & 2.59 & 2.89 & 2.99 & 3 & $1$ & $0.59$ & $<1.5\cdot 10^{-6}$\\
 & & & & & & & & & & & & $0.9$ & $0.24$ & $<1\cdot 10^{-7}$ \\
Dryer & dominant order & 4 & nAND & genera & 322 & 18 & 2.83 & 2.69 & 2.89 & 2.99 & 3 & $1$ & $0.59$ & $<1.5\cdot 10^{-6}$\\
 & & & & & & & & & & & & $0.78$ & $0.19$ & $<1\cdot 10^{-7}$ \\
Dryer & dominant order & 4 & nAND & adj. langs. & 212 & 18 & 2.81 & 2.77 & 2.89 & 2.99 & 3 & $1$ & $0.59$ & $<1.5\cdot 10^{-6}$\\
 & & & & & & & & & & & & $0.55$ & $0.12$ & $<1\cdot 10^{-7}$ \\
Leela & Hindi-Urdu CP & 4 & SOVI & frequency & 320 & 5 & 1.43 & 2.31 & 2.89 & 2.99 & 3 & $0.28$ & $<1.5\cdot 10^{-6}$ & $<1.5\cdot 10^{-6}$\\
 & & & & & & & & & & & & $0.02$ & $<1\cdot 10^{-7}$ & $<1\cdot 10^{-7}$ \\
Leela & Hindi-Urdu & 4 & SOVI & frequency & 311 & 4 & 1.28 & 2.27 & 2.89 & 2.99 & 3 & $0.22$ & $<1.5\cdot 10^{-6}$ & $<1.5\cdot 10^{-6}$\\
 & & & & & & & & & & & & $0.01$ & $<1\cdot 10^{-7}$ & $<1\cdot 10^{-7}$ \\

%% file: tables/Wilcoxon_signed_rank_test_table.tex
$H < H_{Pu}$ & 5 & $3.1\cdot 10^{-4}$ & 0 & $9.8\cdot 10^{-4}$ & 5 & $0.31$ & 5 & $0.08$ & 0 & $3.9\cdot 10^{-3}$ \\
$\left<d\right> < \left<d\right>_{Pu}$ & 0 & $3.1\cdot 10^{-5}$ & 0 & $9.8\cdot 10^{-4}$ & 0 & $0.03$ & 0 & $7.8\cdot 10^{-3}$ & 0 & $3.9\cdot 10^{-3}$ \\
$\left<d\right> < \left<d\right>_{rp}$ & 17 & $6.2\cdot 10^{-3}$ & 1 & $2\cdot 10^{-3}$ & 6 & $0.41$ & 12 & $0.41$ & 0 & $3.9\cdot 10^{-3}$ \\

%% file: appendix.tex
\section{Diversity indices}

\label{app:diversity_indices}

\subsection{Dominance index}

$\bar{S}$ appears in various expressions about $\averageswapdistance$ (Section \ref{sec:introduction}, Appendix \ref{app:mathematical_properties_of_average_swap_distance} and Appendix \ref{app:null_hypotheses}) and thus it is worth asking about its range of variation. 
As $0 \leq p_i \leq 1$, it is easy to see that 
\begin{eqnarray*}0 \leq \bar{S} \leq 1.\end{eqnarray*}
However, 1 is not a tight upper bound for $\bar{S}$ as we will show. In addition, as the value of $n$ increases, languages will have difficulties to fill space of possible permutations. Therefore, we aim for a simple upper bound of $\bar{S}$  that involves just $m$, namely the number of $p_i$'s that are not zero ($m$ is the size of the support set; $m \leq N$). The next property indicates that $S$ is minimized when all orders in the support set are equally likely.
\begin{property}
\begin{eqnarray*}S \geq \frac{1}{m}\end{eqnarray*}
\end{property}

\begin{proof}
Let $q$ be a vector of the same length as $p$. 
By the Cauchy-Schwarz inequality, 
\begin{eqnarray*}\left(\sum_{i=1}^N p_i q_i \right)^2 \leq \sum_{i=1}^N p_i^2 \sum_{i=1}^N q_i^2.\end{eqnarray*}
Suppose that $q_i = 1$ if $p_i >0$ and $q_i = 0$ otherwise. Then, recalling 
that 
\begin{eqnarray*}
\sum_{i = 1}^{N} p_i = 1
\end{eqnarray*}
and noting that 
\begin{eqnarray*}
\left(\sum_{i=1}^N p_i q_i \right)^2 & = & \left(\sum_{i=1}^N p_i \right)^2 \\
   & = & 1 \\
\sum_{i=1}^N q_i^2 & = & \sum_{i=1}^N q_i \\ 
                   & = & m,  
\end{eqnarray*}
the inequality becomes
\begin{eqnarray*}\sum_{i=1}^N p_i^2 \geq \frac{1}{m}.\end{eqnarray*}
\end{proof}
\noindent Then $\bar{S}$ is bounded above by the proportion of non-zero probability orders, i.e.
\begin{eqnarray*}
\dominanceindex \leq 1 - \frac{1}{m}.
\end{eqnarray*}

\subsection{Entropy}

\label{app:subsec_entropy}

It is well-known that \citep{Cover2006a}
\begin{eqnarray*}
0 \leq H \leq H_{max} = \log m.
\end{eqnarray*}
Here we compute $H$ using the plug-in estimator, namely 
\begin{eqnarray}
\label{eq:entropy_plug-in_estimator}
H & = & -\sum_{i=1}^{N} \frac{f_i}{F}\log\frac{f_i}{F} \nonumber \\
\label{eq:entropy_plug-in_estimator_expanded}
  & = & \log F - \frac{1}{F}\sum_{i=1}^{N} f_i \log f_i.
\end{eqnarray}
In the previous expression, $f_i \log f_i = 0$ if $f_i = 0$ follows from $p_i \log p_i = 0$ if $p_i = 0$, which is a convention in information theory \citep{Cover2006a} that can be justified by   
$$\lim_{x \rightarrow 0} x \log x = 0$$ 
thanks to l'H\^{o}pital's rule.
We use the plug-in entropy estimator for three reasons: (a) simplicity, (b) our main goal is not to estimate the true value of entropy but to determine if entropy is significantly small and (c) the plug-in estimator is very strongly correlated with advanced entropy estimators \citep{Bentz2017a}.

\subsection{A swap distance score}

We define the mean swap distance conditioning on a source order $i$ as
\begin{eqnarray}
\left< d | i\right> = \sum_{j=1}^{N} d_{ij} p_j.
\label{eq:average_swap_distance_from_source_order}
\end{eqnarray}
One could investigate swap distance minimization with $\left< d | i\right>$ and choosing $i$ as the source order. The choice of the source could be determined by that ordering being canonical, dominant or underlying word order. However, using a single word order as source assumes that all word order variation is produced from it.
That is, other word orders, specially high frequency ones, could be sources from which word order variation is produced. 
Then we define the mean swap distance as the mean swap distance conditioning on all possible source orders, i.e. 
\begin{eqnarray}
\averageswapdistance & = & \sum_{i=1}^{N} p_i \left< d | i \right> \nonumber \\
                & = & \sum_{i=1}^{N} \sum_{j=1}^{N} d_{ij} p_i p_j. \label{eq:expanded_average_conditional_probability}
\end{eqnarray}
Thus $\averageswapdistance$ can be seen as a particular case of the expectation of $d_{ij}$ with respect to the joint probability of orders $i$ and $j$, i.e. 
\begin{eqnarray*}
\averageswapdistance = \sum_{i=1}^{N} \sum_{j=1}^{N} d_{ij} p_{ij}, 
\end{eqnarray*}
that is obtained assuming that orders $i$ and $j$ are statistically independent and then $p_{ij} = p_i p_j$. 
$\averageswapdistance$ is a weighted average distance from every order to every other order, as if each order could potentially radiate to all other word orders following trajectories in the permutohedron. With respect to using just $\left< d | i\right>$, $\averageswapdistance$ solves the problem of the choice of the most frequent word order as the canonical in case the canonical order is unclear or weak as it happens in the case of couples of primary alternating word orders \citep{wals-81}.

\section{Graph theory and combinatorics}

\label{app:theory} 

\subsection{Graph theory reminder}
\label{subsect:graph_theory_reminder}

We review some elementary concepts from graph theory. The degree of a vertex is its number of edges. A regular graph is a graph where all vertices have the same degree. A Hamiltonian graph is a graph that has a path that visits every vertex once.
We define $d_{ij}$ as the number of edges of the shortest path connecting vertices $i$ and $j$ in a graph. $d_{ij}$ is the topological distance, or shortly, the distance between vertices $i$ and $j$.
Then a vertex is at distance zero of itself ($d_{ii}=0$), two connected vertices are at distance one ($d_{ij}=1$ if $i$ and $j$ form an edge) and so on. The diameter of a graph of $N$ vertices is the maximum distance between vertices, i.e.  
\begin{equation*}
d_{max} = \max_{1 \leq i,j \leq N} d_{ij}.
\end{equation*}

We define two kinds of mean topological distances: the mean topological distance over all ordered pairs of vertices, that is defined as 
\begin{eqnarray}
\averageswapdistance_{op} = \frac{1}{N^2}\sum_{i=1}^{N} \sum_{j=1}^{N} d_{ij},
\label{eq:average_topopological_distance_all_ordered_pairs}
\end{eqnarray}
and the mean topological distance over all unordered pairs (excluding pairs formed by the same vertex), that is defined as
\begin{eqnarray}
\averageswapdistance_{up} = \frac{1}{{N \choose 2}}\sum_{i=1}^{N} \sum_{j=i+1}^{N} d_{ij}.
\label{eq:average_topopological_distance_all_unordered_pairs}
\end{eqnarray}
Since
\begin{eqnarray*}
\sum_{i=1}^{N} \sum_{j=i+1}^{N} d_{ij} 
   & = & \frac{1}{2}\left(\sum_{i=1}^{N} \sum_{j=1}^{N} d_{ij} - \sum_{i=1}^N d_{ii} \right) \\
   & = & \frac{1}{2} \sum_{i=1}^{N} \sum_{j=1}^{N} d_{ij},
\end{eqnarray*}
$\averageswapdistance_{up}$ can be expressed equivalently as
\begin{eqnarray}
\averageswapdistance_{up} & = & \frac{1}{2{N \choose 2}}\sum_{i=1}^{N} \sum_{j=1}^{N} d_{ij} \nonumber \\
                      & = & \frac{1}{N(N-1)}\sum_{i=1}^{N} \sum_{j=1}^{N} d_{ij}. 
                      \label{eq:average_topopological_distance_all_unordered_pairs_bis}
\end{eqnarray}

\subsection{The permutohedron}

\label{subsec:permutohedron}

The permutohedron of order $n$ is a graph where vertices are permutations of 
$$1,2,\dots,n$$ 
and an edge links two permutations if swapping two consecutive elements in one of the permutations leads to the other permutation. When $n=3$, one obtains the permutation ring that has been used in word order research \citep{Ferrer2008e,Ferrer2013e,Ferrer2016c,Ferrer2023a} (Figure \ref{fig:word_order_permutation_ring}). When $n=4$, one obtains the graph in Figure \ref{fig:permutohedron_3_and_4}.
The permutohedron has $N = n!$ vertices and is a regular graph where vertex degree is $n - 1$. Thus it has 
\begin{eqnarray*}\frac{n - 1}{2} n!\end{eqnarray*} 
edges. The permutohedron is also a Hamiltonian graph \footnote{The demonstration relies on algorithms that are able to generate all permutations (i.e. all vertices of the graph) by swapping a pair of adjacent elements to get to the next vertex in the path that visits every vertex once \citep{Sedgewick1977a}. Every swap is equivalent to following an edge in the permutohedron.}

\begin{figure}
\caption{\label{fig:permutohedron_3_and_4} Permutohedra of order $n$ with $n = 3$ (top) and $n=4$ (bottom). Every vertex is a distinct permutation of the integer numbers between $1$ and $n$. An edge joining two vertices indicates that one vertex yields the other vertex by swapping a pair of adjacent numbers. 
}
\centering
\includegraphics[width = 0.6\textwidth]{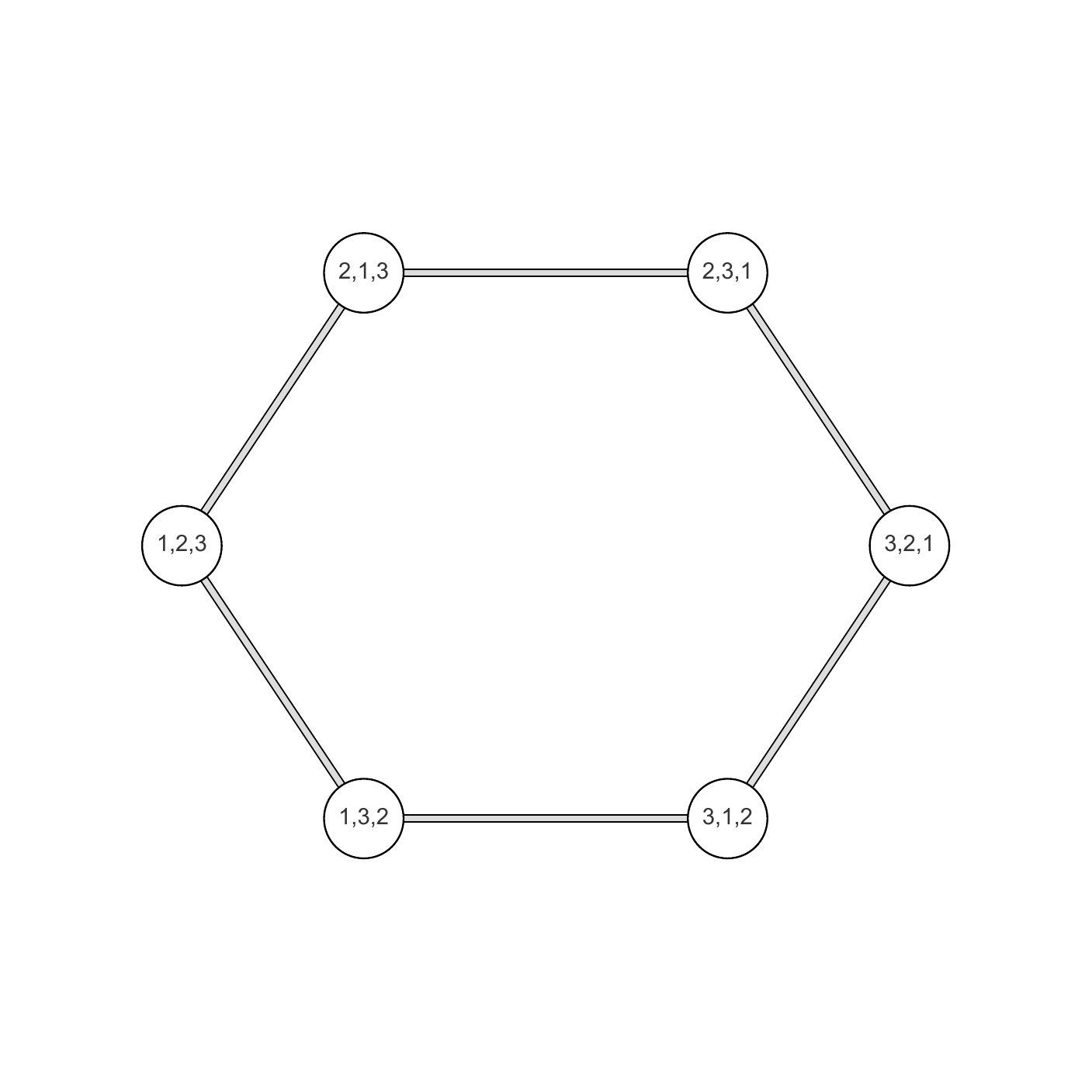}\\
\includegraphics[width = 0.6\textwidth]{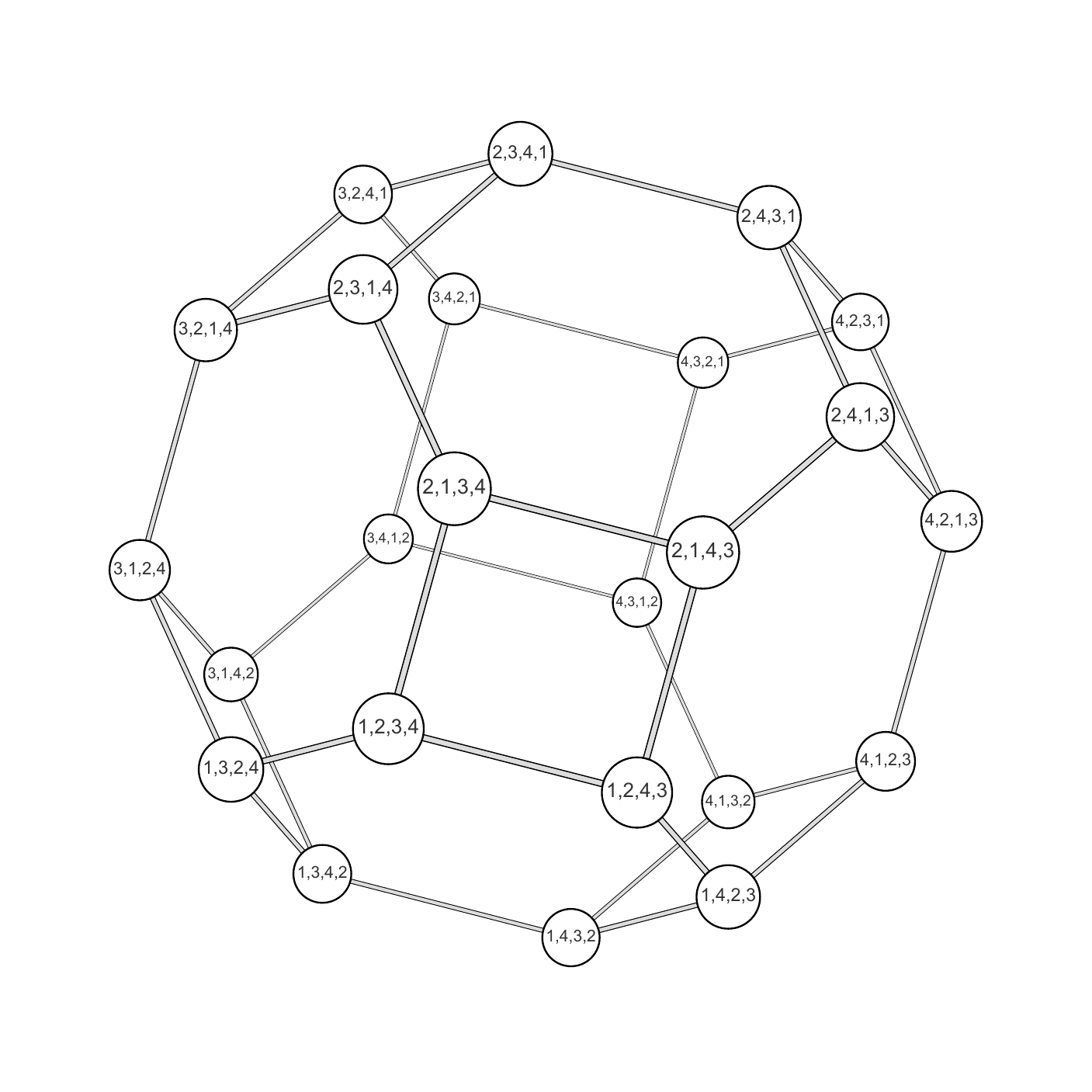}
\end{figure}

We define $d$, the swap distance between two permutations, as their distance in the permutohedron. By definition, $d \in [0,d_{max}]$, where $d_{max}$ is the diameter of the permutohedron.
It is well-known that \citep{Ceballos2015a}
\begin{eqnarray}
d_{max} =  {n \choose 2}.    
\label{eq:diameter}
\end{eqnarray}

We define $T(n, k)$ as the number of permutations of a string of length $n$ at swap distance $k$. 
By definition,
\begin{eqnarray}
\sum_{k\geq 0} T(n, k) = n!.
\label{eq:factorial_property}
\end{eqnarray}
$T(n, k)$ is also known as a Mahonian number \citep{MacMahon1913a}. 
Mahonian numbers can be defined in high level of detail as 
\citep{OEIS_A008302}\footnote{See the section {\em Formula} and the note ``From Andrew Woods, Sep 26 2012, corrected by Peter Kagey, Mar 18 2021''.},  
\begin{eqnarray*}
T(n, k) = \begin{cases}
1 & \text{if }n = 1 \text{ and } k = 0\\
0 & \text{if } n < 0, k < 0 \text{ or } k > {n \choose 2} \\
\sum_{j=0}^{n-1} T(n-1, k-j) & \text{otherwise} \\
T(n, k-1) + T(n-1, k) - T(n-1, k-n) & \text{otherwise.}
\end{cases}
\end{eqnarray*}
Notice that the previous equation has two equivalent recursive cases marked with ``otherwise''.
The definition can be formulated compactly for $n\geq 1$ as 
\begin{eqnarray}
T(n, k) = \begin{cases}
1 & \text{if }n = 1 \text{ and } k = 0\\
0 & \text{if }n = 1 \text{ and } k \neq 0\\
\sum_{j = 0}^{n-1} T(n - 1, k - j) & \text{otherwise}.
\end{cases}\label{eq:Kagey}
\end{eqnarray}
See a proof in Appendix \ref{app:Mahonian_numbers}.

The next property will help us to obtain the value of $\averageswapdistance_{up}$ and $\averageswapdistance_{op}$.

\begin{property}
\label{prop:sum_of_distances}
The Mahonian numbers satisfy
\begin{eqnarray}
\sum_{k \geq 0} T(n, k) \cdot k = \frac{n(n-1)}{4}n!.
\label{eq:sum_of_distances}
\end{eqnarray}
\end{property}

\begin{proof}
We will prove this property by induction on $n$ is as follows.
\begin{itemize}
\item
{\em Base case, $n = 1$.} Trivially, 
\begin{eqnarray*}\sum_{k \geq 0} T(1, k) \cdot k = T(1, 0) \cdot 0 = 1 \cdot 0 = 0 = \frac{1 \cdot 0}{4} 1!.
\end{eqnarray*}
\item 
{\em Induction hypothesis}. Equation \ref{eq:sum_of_distances}.
\item
{\em Induction step.} We aim to show that 
\begin{eqnarray*}\sum_{k \geq 0} T(n+1, k) \cdot k = \frac{n(n+1)}{4}(n+1)!\end{eqnarray*}
for $n > 1$. 
\end{itemize}
As for the induction step, the well-known recursive definition of Mahonian numbers (Equation \ref{eq:Kagey}),
in particular 
\begin{eqnarray*}T(n, k) = \sum_{j=0}^{n-1} T(n-1, k-j),\end{eqnarray*}
leads to 
\begin{align*}
    \sum_{k \geq 0}T(n+1, k)\cdot k &= \sum_{k \geq 0} \left[ \sum_{j = 0}^n T(n, k-j) \right] \cdot k & \text{def. }  T(n, k)
\end{align*} 
Then, a series of mechanical algebraic manipulations yield
\begin{align*}    
    &= \sum_{j = 0}^n \sum_{k \geq 0} T(n, k-j) \cdot k & \text{swap summations}\\ 
    &= \sum_{j = 0}^n \sum_{k \geq j} T(n, k-j) \cdot k & T(n, k) = 0 \text{ for negative }k\\
    &= \sum_{j = 0}^n \sum_{k \geq 0} T(n, k) \cdot (k+j) & \text{change of variable } (k \leftarrow k - j)\\
    &= \sum_{j = 0}^n \sum_{k \geq 0} T(n, k) \cdot k+\sum_{j = 0}^n \sum_{k \geq 0} T(n, k) \cdot j & \text{split}\\
    &= \sum_{j = 0}^n \left[\frac{n(n-1)}{4}n!\right]+\sum_{j = 0}^n n! \cdot j & \text{induction hypothesis and Equation \ref{eq:factorial_property}}\\
    &= \left[\frac{n(n-1)}{4}n!\right] \sum_{j = 0}^n 1 + n!\sum_{j = 0}^n j & \\    
    &= \left[\frac{n(n-1)}{4}n!\right] (n+1) + n! \frac{n (n+1)}{2} & \\    
    &= \frac{n(n-1)}{4}(n+1)!+\frac{n}{2} (n + 1)!\\
    &= \frac{n(n+1)}{4}(n + 1)!.
\end{align*}
\end{proof}

\begin{property}
In a permutohedron of order $n$, the mean topological distances are 
\begin{eqnarray*}\averageswapdistance_{op} = \frac{1}{2}d_{max}\end{eqnarray*}
and
\begin{eqnarray*}\averageswapdistance_{up} = \frac{N}{2(N-1)}d_{max}.\end{eqnarray*}
\label{prop:mean_topological_distance}
\end{property}
\begin{proof}
Thanks to Property \ref{prop:sum_of_distances},
for any vertex $i$, the sum of the minimum distances from $i$ to any other vertex is 
\begin{eqnarray*}
\sum_{j = 1}^{N} d_{ij} = \sum_{k \geq 0} T(n, k) \cdot k = \frac{n (n - 1)}{4} N = \frac{d_{max}}{2}N
\end{eqnarray*}
thanks to the definition of $d_{max}$ (Equation \ref{eq:diameter}). Then
\begin{eqnarray*}
\sum_{i = 1}^{N}\sum_{j = 1}^{N} d_{ij} = \frac{d_{max}}{2} N^2.
\end{eqnarray*}
Recalling the definition of $\averageswapdistance_{op}$ (Equation \ref{eq:average_topopological_distance_all_unordered_pairs_bis}), we obtain
\begin{eqnarray*}\averageswapdistance_{op} = \frac{1}{2}d_{max},\end{eqnarray*}
while recalling the parallel definition of $\averageswapdistance_{up}$ (Equation \ref{eq:average_topopological_distance_all_unordered_pairs_bis}), we obtain
\begin{eqnarray*}\averageswapdistance_{up} = \frac{N}{2(N-1)} d_{max}.\end{eqnarray*}

\end{proof}
We define $\bar{i}$ as the reverse of order $i$. Then $d_{\bar{i}i} = d_{max}$.
The following property gives an exact relationship between $d_{ij}$ and $d_{\bar{i}j}$.
\begin{property}
\begin{eqnarray*}d_{\bar{i}j} = d_{max} - d_{ij}.\end{eqnarray*}
\label{prop:complementary_distance}
\end{property}
\begin{proof}
Consider a permutation $\pi=\pi_1,\dots,\pi_a,\dots,\pi_n$ of the integers in $[1,n]$.
$inv(\pi)$, the number of inversions of $\pi$, is defined as the number of pairs $a$ and $b$, such that $1 \leq a < b \leq n$ and $\pi_a > \pi_b$. It is easy to see that
\begin{eqnarray*}inv(i) + inv(\bar{i}) = {n \choose 2} = d_{max}.\end{eqnarray*}
We use $I$ to refer to the identity permutation, i.e. $I = 1,2,3,\dots,n$.
It is well-known that the number of inversions is equivalent to $d(\pi)$, the swap distance of $\pi$ to the identity permutation. \footnote{To understand the equivalence between swap distance and inversions, notice first that the number of inversions of a permutation equals the minimal number of simple transpositions, namely transpositions of adjacent elements, that are needed to write the permutation \url{https://www.findstat.org/StatisticsDatabase/St000018/}. In turn, the number of simple transpositions is our definition of swap distance. Indeed, the swap distance is the Coxeter length in the symmetric group, which is known to be equivalent to the number of inversions of a permutation \cite[Proposition 1.5.2]{Bjorner2005a}.
}
Thus
\begin{eqnarray*}d_{Ii} + d_{I\bar{i}} = {n \choose 2}.\end{eqnarray*}
If one takes vertex $j$ of the permutohedron as the reference permutation instead of the identity permutation, one obtains the desired result, i.e. 
\begin{eqnarray*}d_{ji} + d_{j\bar{i}} = {n \choose 2}.\end{eqnarray*}
The last step follows by the vertex transitivity of the permutohedron. 
\end{proof}

\subsection{Mahonian numbers}

\label{app:Mahonian_numbers}

Here we revisit a well-known property of Mahonian numbers that we have not found proven in standard publications. 
\begin{property}
$T(n, k)$ can be defined as
\begin{eqnarray*}
T(n, k) = \begin{cases}
1 & \text{if }n = 1 \text{ and } k = 0\\
0 & \text{if }n = 1 \text{ and } k \neq 0\\
\sum_{j = 0}^{n-1} T(n - 1, k - j) & \text{otherwise}.
\end{cases}
\end{eqnarray*}
\end{property}
\begin{proof}

Here we provide a proof for the statement of the property. 
First, the Mahonian numbers $T(n, k)$ are defined as the coefficients of the expansion of the product \citep{OEIS_A008302}
\begin{eqnarray*}
\prod_{k = 0}^{n-1} \sum_{j = 0}^k x^j,
\end{eqnarray*}
namely 
\begin{eqnarray*}\prod_{k = 0}^{n-1} \sum_{j = 0}^k x^j = \sum_i T(n, k) x^k.\end{eqnarray*}
For the case $n = 1$, the LHS equals $1$, which implies for the RHS that $T(1, 0) = 1$ and $T(1, k) = 0$ for every other $k \neq 0$. For $n > 1$,  
\begin{align*}
\sum_k T(n, k) x^k &= \prod_{k = 0}^{n - 1} \sum_{j = 0}^{k} x^j\\ 
                       &= \left(\prod_{k = 0}^{n - 2} \sum_{j = 0}^{k} x^j\right) \cdot \left(\sum_{j = 0}^{n-1} x^j\right) & \text{split} \\
                       &= \left(\sum_k T(n-1, k) x^k\right) \cdot \left(\sum_{j = 0}^{n-1} x^j\right) & \text{substitution} \\
                       &= \sum_k \sum_{j = 0}^{n-1} T(n - 1, k) x^{k+j}\\
                       &= \sum_k \sum_{j = 0}^{n-1} T(n - 1, k - j) x^k & \text{change of variable}
\end{align*}
and then, as the coefficients of the polynomials on both sides of the equality must be the same, we finally obtain
\begin{eqnarray*}
T(n, k) = \sum_{j = 0}^{n-1} T(n - 1, k - j).
\end{eqnarray*}
\end{proof}

\section{Properties of the new diversity index}

\label{app:mathematical_properties_of_average_swap_distance}

The following property indicates that the value of $\left< d| i\right>$ is the complementary of that of $\left< d| \bar{i}\right>$ with respect to the diameter.
\begin{property}
\begin{eqnarray*}
\left< d| i\right> + \left< d| \bar{i}\right> = d_{max},     
\end{eqnarray*}
\label{prop:symmetries}
\end{property}
\begin{proof}
Property \ref{prop:complementary_distance} transforms
Equation \ref{eq:average_swap_distance_from_source_order}
into 
\begin{eqnarray*}
\left< d| \bar{i}\right> & = & \sum_{j=1}^N p_j (d_{max} - d_{ij}) \\
                         & = & d_{max} - \left< d| i\right>.
\end{eqnarray*} 
\end{proof}

$\averageswapdistance$ can be rewritten equivalently as
\begin{eqnarray}
\averageswapdistance = \sum_{d=0}^{d_{max}} P(d)d = \sum_{d=1}^{d_{max}} P(d)d,
\label{eq:alternative_average_swap_distance}
\end{eqnarray}
where $P(d)$ is the probability mass of a certain distance $d$, i.e.  
\begin{eqnarray*}
P(c) = \sum_{i=1}^{N} \sum_{\substack{j=1 \\ d_{ij} = c}}^{N}
p_i p_j.
\end{eqnarray*}
Obviously,
\begin{eqnarray*}    
\sum_{d=0}^{d_{max}} P(d) = 1.
\end{eqnarray*}
Notice that  
\begin{eqnarray}
P(0) & = & \sum_{i=1}^{N} p_i \sum_{\substack{j=1 \\ d_{ij} = 0}}^{N} 
p_j \nonumber \\
     & = & \sum_{i=1}^{N} p_i^2 = S, \label{eq:P0_and_Simpson_index}
\end{eqnarray}
where $S$ is the Simpson index (Equation \ref{eq:Simpson_index}).
Notice also that $P(0) > 0$ because all the $p_i$'s cannot be zero.

\subsection{Lower and upper bounds of the average swap distance score}

\label{subsec:mathematical_properties_of_average_swap_distance}

\subsubsection{The minimum value of $\averageswapdistance$}

A tight lower bound of $\averageswapdistance$ is straightforward. 

\begin{property}
$0 \leq \averageswapdistance$ with equality if and only if just one of the $p_i$'s is non-zero. 
\end{property}
\begin{proof}
~\\
\indent 
Step 1. Show that $\averageswapdistance \geq 0$. Trivial since $0 \leq d_{ij}$ by definition.

Step 2. Show that $\averageswapdistance = 0$ if and only if all probability mass is on just one order. As $d_{ij} = 0$ if and only if $i=j$, then Equation \ref{eq:expanded_average_conditional_probability} indicates that 
$\averageswapdistance > 0$ if and only if $p_i p_j >0$ for some $i\neq j$.
\end{proof}

\subsubsection{An upper bound of $\averageswapdistance$}

Since the distance to a source order is a number $d$ such that $0 \leq d \leq d_{max}$, where $d_{max}$ is the diameter of the permutohedron, it follows trivially that $0 \leq \averageswapdistance \leq d_{max}$. However, we will show that $d_{max}$ is 
an upper bound that is too loose. Recall that the dominance index is the complementary of the Simpson index, i.e. $\dominanceindex = 1 - S$, with $S$ defined as in Equation \ref{eq:Simpson_index}).

\begin{property}
\begin{eqnarray*}
\averageswapdistance \leq d_{max} \dominanceindex,
\end{eqnarray*}
where $d_{max}$ is the diameter and $\dominanceindex$ is the dominance index.
\label{prop:upper_bound_Simpson_index}
\end{property}
\begin{proof}
Equation \ref{eq:alternative_average_swap_distance} leads to 
\begin{eqnarray*}
\averageswapdistance & = & \sum_{d=1}^{d_{max}-1} P(d)d + d_{max} \left(1 - \sum_{d=0}^{d_{max} - 1} P(d) \right) \\
  & = & d_{max}(1-P(0)) - \sum_{d=1}^{d_{max}-1} (d_{max} - d) P(d) \\
  & \leq & d_{max}(1 - P(0)). \\
\end{eqnarray*} 
Recalling that $P(0) = S$ (Equation \ref{eq:P0_and_Simpson_index}) and $\dominanceindex = 1 - S$, we conclude 
$\averageswapdistance \leq d_{max}\dominanceindex.$
\end{proof}

\subsubsection{The maximum value of $\averageswapdistance$}

$H$ achieves its maximum value when all orders are equally likely. We will show first that the value of $\averageswapdistance$ achieved in that situation is also achieved by another configuration according to the following property.
Later on, we will conjecture that this value is also maximum for $\averageswapdistance$.
 
\begin{property}
\label{prop:average_swap_distance equally_likely_orders_and_more}
\begin{eqnarray*}\averageswapdistance = \frac{1}{2}d_{max}\end{eqnarray*}
in at least two configurations 
\begin{itemize}
\item 
{\em Configuration 1.} $p_i = 1/N$ for any $i$ such that $1 \leq i \leq N$. 
\item 
{\em Configuration 2.} $p_i = 0$ except for two orders, $k$ and $l$, such that $d_{kl} = d_{max}$ and $p_k = p_l = 1/2$.
\end{itemize}
\end{property}
\begin{proof}
For the 1st configuration, Equations \ref{eq:expanded_average_conditional_probability} and \ref{eq:average_topopological_distance_all_ordered_pairs} and Property \ref{prop:mean_topological_distance} give
\begin{eqnarray*}
\averageswapdistance = d_{op} = \frac{1}{2}d_{max}.
\end{eqnarray*}
For the 2nd configuration, Equation \ref{eq:expanded_average_conditional_probability} gives 
\begin{eqnarray*}\averageswapdistance = 2 d_{max} p_k p_l = \frac{1}{2}d_{max}.\end{eqnarray*}

\end{proof}

$\averageswapdistance$ can be defined as product involving the vector of probabilities 
$$p=(p_1,\dots,p_i,\dots,p_N)$$ 
and $D = \left\{d_{ij} \right\}$, the matrix of vertex-vertex distances in the permutohedron, i.e.
\begin{eqnarray}
\averageswapdistance = p^T D p.
\label{eq:average_swap_distance_in_matrix_algebra}
\end{eqnarray}
Recall that $\sum_i p_i = 1$.

In that setting, the following lemma will help us to obtain tight upper bounds of $\averageswapdistance$.

\begin{lemma} 
\label{lem:maximum_average_swap_distance}
Let $Q$ be a real symmetric matrix of size $N \times N$ and let $k$ be a real constant. Consider the quadratic optimization problem of the form $\mathbf{max}_{x\in \mathbb{R}^N}\ x^T Q x$ under the constraint $\sum x_i = 1$. 
Let $O$ be the all-ones matrix of the same size as $Q$. 
If the matrix $Q - Ok$ is negative semidefinite then  
\begin{eqnarray*}\mathbf{max}_{x\in \mathbb{R}^N}\ x^T Q x \leq k.\end{eqnarray*}
\end{lemma}
\begin{proof}
    We have that 
    \begin{eqnarray*}\mathbf{max}\ x^T (Q - Ok) x = \mathbf{max}\ x^T Q x - k,\end{eqnarray*} 
    because $x^T (Ok) x = k x^T O x = k$ thanks to $\sum x_i = 1$. \\
    If $(Q - Ok)$ is a negative semidefinite matrix, by definition of negative semidefinite, we have that for every $x$, $x^T(Q - Ok)x \leq 0$, which implies that $\mathbf{max}\ x^T (Q - Ok) x =  \mathbf{max}\ x^T Q x - k \leq 0$ and finally $\mathbf{max}_{x\in \mathbb{R}^N}\ x^T Q x \leq k$.
\end{proof}

We conjecture that 
\begin{eqnarray*}
\averageswapdistance \leq \frac{d_{max}}{2}
\end{eqnarray*}
for any $n \geq 1$. 
If the conjecture is true it will always yield a tight upper bound thanks to Property \ref{prop:average_swap_distance equally_likely_orders_and_more}
The following property states that this is true at least up to $n = \nmax$. 
\begin{property}
\label{prop:global_upper_bound}
For $1 \leq n \leq \nmax$,
\begin{eqnarray*}
\averageswapdistance \leq \frac{d_{max}}{2}
\end{eqnarray*}
with equality if the $p_i$'s satisfy one of the conditions indicated in Property \ref{prop:average_swap_distance equally_likely_orders_and_more}.
\end{property}

\begin{proof}

The conjecture is trivially true for $n \leq 2$ (recall Property \ref{prop:average_swap_distance equally_likely_orders_and_more}).
For $n > 2$, the proof consists of two steps
\begin{enumerate}
\item
Apply Property \ref{prop:average_swap_distance equally_likely_orders_and_more} to obtain configurations such that $\averageswapdistance = d_{max}/2$.
\item 
\label{it:step2_lower_bound}
Show that $\averageswapdistance \leq d_{max}/2$. The argument is based on the equivalence between a negative semi-definite matrix and a matrix whose all eigenvalues are non-positive. 
The procedure consists of the following steps: 
\begin{enumerate}
\item  
Computing $D$ by means of a breadth-first traversal of the permutohedron of order $n$ from an initial arbitrary order. 
\item
Computing the eigenvalues of the matrix $D'=D - \frac{d_{max}}{2}O$, where $O$ is the all-one matrix of the same size as $D$.
\item 
Checking that $D'$ is negative semi-definite by checking that all its eigenvalues are non-positive. We compute the eigenvalues of $D'$ by means of the Implicitly Restarted Arnoldi Method (IRAM). \footnote{In particular, we called the function {\tt scipy.linalg.eigvals} available in the {\tt scipy} Python library. The function is a wrapper for the ARPACK implementation of the method. } 
\item 
Invoking Lemma \ref{lem:maximum_average_swap_distance}, to conclude that $d_{max}/2$ is an upper bound of $\averageswapdistance$.
\end{enumerate}
\end{enumerate}
Next we apply Step \ref{it:step2_lower_bound} to specific values of $n$. When $n=3$, $d_{max}/2 = 3/2$, $D$ can be defined as 
\begin{eqnarray*}
D = \begin{pmatrix} 
0 & 1 & 1 & 2 & 2 & 3 \\
1 & 0 & 2 & 1 & 3 & 2\\
1 & 2 & 0 & 3 & 1 & 2\\
2 & 1 & 3 & 0 & 2 & 1\\
2 & 3 & 1 & 2 & 0 & 1\\
3 & 2 & 2 & 1 & 1 & 0\end{pmatrix}
\end{eqnarray*}
and then the eigenvalues of the matrix
\begin{eqnarray*}
D - \frac{3}{2}O = \begin{pmatrix} 
- \frac{3}{2} & - \frac{1}{2} & - \frac{1}{2} & \frac{1}{2} & \frac{1}{2} & \frac{3}{2} \\
- \frac{1}{2} & - \frac{3}{2} & \frac{1}{2} & - \frac{1}{2} & \frac{3}{2} & \frac{1}{2}\\
- \frac{1}{2} & \frac{1}{2} & - \frac{3}{2} & \frac{3}{2} & - \frac{1}{2} & \frac{1}{2}\\
\frac{1}{2} & - \frac{1}{2} & \frac{3}{2} & - \frac{3}{2} & \frac{1}{2} & - \frac{1}{2}\\
\frac{1}{2} & \frac{3}{2} & - \frac{1}{2} & \frac{1}{2} & - \frac{3}{2} & - \frac{1}{2}\\
\frac{3}{2} & \frac{1}{2} & \frac{1}{2} & - \frac{1}{2} & - \frac{1}{2} & - \frac{3}{2}
\end{pmatrix}
\end{eqnarray*}
are
\begin{eqnarray*}
\lambda_1 = \lambda_2 = & -4\\
\lambda_3 = & -1\\
\lambda_4 = \lambda_5 = \lambda_6 = & 0.
\end{eqnarray*}
Table \ref{tab:eigenvalues} shows that the eigenvalues of the matrix $D - Od_{max}/2$ are all negative for $3 \leq n \leq \nmax$.

\begin{table}
\caption{\label{tab:eigenvalues}
The eigenvalues of the matrix $D - Od_{max}/2$ as a function of $n$, the order of the permutohedron.}
\centering
\begin{tabular}{@{} lllll @{}}
\hline
&   & \multicolumn{3}{c}{eigenvalues} \\
\cmidrule(lr){3-5} 
$n$ & $d_{max}/2$ & group 1 & group 2 & group 3 \\
\hline
3   & 1.5 & $\lambda_{1,2} = -4$ & $\lambda_3 = -1$ & $\lambda_{4,5,6} = 0$ \\
4   & 3 & $\lambda_{1,2,3} = -20$ & $\lambda_{4,5,6} = -4$ & 
$\lambda_{7,\dots,24} = 0$ \\
5   & 5 & $\lambda_{1,2,3,4} = -120$ & $\lambda_{4,\dots,10} = -20$ & $\lambda_{11,\dots,120} = 0$ \\
6   & 7.5 & $\lambda_{1,\dots,5} = -840$ & $\lambda_{6,\dots,15} = -120$ & $\lambda_{16,\dots,720} = 0$ \\
7   & 10.5 & $\lambda_{1,\dots,6} = -6720$ & $\lambda_{7,\dots,21} = -840$ & $\lambda_{22,\dots,5040} = 0$ \\
\hline
\end{tabular}
\end{table}

\end{proof}

\subsection{A compact expression for $\averageswapdistance$ when $n=3$.}


\label{app:complementary_results_triplets}

Next we will introduce a new expression for $\averageswapdistance$ for $n=3$ that is computationally efficient and that leads to straightforward tight upper bounds for $\averageswapdistance$.

\begin{property}

In triplets, $\averageswapdistance$ can be expressed equivalently as
\begin{eqnarray*} 
\averageswapdistance =  \frac{3 - \Delta}{2},
\end{eqnarray*}
where $\Delta$ can be defined in two equivalent ways. First,
\begin{eqnarray*}\Delta = 3[P(0) - P(3)] + P(1) - P(2).\end{eqnarray*}
Second, let $q_i = p_i - p_{\bar{i}}$ and assume that vertices are labelled with numbers from 1 to 6 starting from some arbitrary vertex of the permutohedron (Figure \ref{fig:permutohedron_3_and_4} top) and continuing in a clockwise sense (or anticlockwise sense).
Then 
\begin{eqnarray}
\Delta & = & 3 \sum_{i=1}^3 q_i^2 - 2(q_1 q_3 - q_1 q_2 - q_2 q_3) \nonumber \\ 
       & = & \sum_{i=1}^3 q_i^2 + (q_1 + q_2)^2 + (q_2 + q_3)^2 + (q_1 - q_3)^2. \label{eq:fancy_Delta}
\end{eqnarray}

\end{property}
\begin{proof} 
In triplets, 
$\averageswapdistance$ is
\begin{eqnarray*}\averageswapdistance_a = P(1) + 2P(2) + 3P(3).\end{eqnarray*} 
Since 
\begin{eqnarray*}
\averageswapdistance + 2 P(1) + P(2) & = & 3 [P(1) + P(2) + P(3)] \\    
  & = & 3[1 - P(0)],
\end{eqnarray*}
$\averageswapdistance$ is also
\begin{eqnarray*}\averageswapdistance_b = 3[1 - P(0)] - 2P(1) - P(2).\end{eqnarray*} 
Therefore 
\begin{eqnarray*}
\averageswapdistance & = & \frac{1}{2}(\averageswapdistance_a + \averageswapdistance_b) \\
                     & = & \frac{1}{2}(3 - \Delta), 
\end{eqnarray*}
where \begin{eqnarray*}\Delta = 3[P(0) - P(3)] + P(1) - P(2).\end{eqnarray*}

In triplets, the definition of $\averageswapdistance$ can be split as
\begin{eqnarray*}
\averageswapdistance_1 & = & \sum_{i=1}^3 p_i \left<d | i \right> + \sum_{i=4}^6 p_i \left<d | i \right>.
\end{eqnarray*}
Applying Property \ref{prop:symmetries} to the second summation, $\averageswapdistance$ can be expressed as
\begin{eqnarray*}
\averageswapdistance_1 & = & \sum_{i=1}^3 p_i \left<d | i \right> + \sum_{i=4}^6 p_i (3 - \left<d | \bar{i} \right>)
\end{eqnarray*}
Assuming the vertex labelling indicated in the statement,
we obtain
\begin{eqnarray*}
\averageswapdistance_1 & = & \sum_{i=1}^3 p_i \left<d | i \right> + \sum_{i=1}^3 p_{\bar{i}} (3 - \left<d | i \right>) \\
  & = & \sum_{i=1}^3 \left[(p_i - p_{\bar{i}}) \left<d | i \right> + 3 p_{\bar{i}} \right].
\end{eqnarray*}
Analogously, the application of Property \ref{prop:symmetries} to the first summation, $\averageswapdistance$ yields
\begin{eqnarray*}
\averageswapdistance_2 & = & \sum_{i=1}^3 p_i (3 - \left<d | i \right>) + \sum_{i=4}^6 p_i \left<d | i \right> \\
& = & \sum_{i=1}^3 p_i (3 - \left<d | i \right>) + \sum_{i=1}^3 p_{\bar{i}} \left<d | \bar{i} \right> \\
  & = & \sum_{i=1}^3 \left[(p_{\bar{i}} - p_i) \left<d | \bar{i} \right> + 3 p_i \right].
\end{eqnarray*}
Let 
\begin{equation*}
\Delta' = \sum_{i=1}^3 (p_i - p_{\bar{i}}) (\left<d | i \right> - \left<d | \bar{i} \right>) .
\end{equation*}    
Then 
\begin{eqnarray*}
\averageswapdistance_1 + \averageswapdistance_2 & = & \Delta' + 3\sum_{i=1}^3 (p_i + p_{\bar{i}}) \\
& = & \Delta' + 3\sum_{i=1}^6 p_i \\
& = & \Delta' + 3.
\end{eqnarray*}
Finally, the fact that $\averageswapdistance_1 = \averageswapdistance_2$, 
leads to a new expression for $\averageswapdistance$ as
\begin{eqnarray*}
\averageswapdistance & = &\frac{1}{2}(\averageswapdistance_1 + \averageswapdistance_2) \\
& = & \frac{3 + \Delta'}{2},
\end{eqnarray*}
where  
\begin{eqnarray*}
\Delta' & = & (p_1 - p_4)[p_2 + 2p_3 + 3p_4 + 2p_5 + p_6 - (3p_1 + 2p_2 + p_3 + p_5 + 2p_6)] + \\ 
        &  & (p_2 - p_5)[p_1 + p_3 + 2p_4 + 3p_5 + 2p_6 - (2p_1 + 3p_2 + 2p_3 + p_4 + p_6)] + \\
        &  & (p_3 - p_6) [2p_1 + p_2 + p_4 + 2p_5 + 3p_6 - (p_1 + 2p_2 + 3p_3 + 2p_4 + p_5)] \\
  & = & -3(p_1 - p_4)^2 - 3(p_2 - p_5)^2 - 3(p_3 - p_6)^2 + \\ 
 &  & (p_1 - p_4)[(p_3 - p_6) - (p_2 - p_5)] + \\ 
 &  & (p_2 - p_5)[-(p_1 - p_4) - (p_3 - p_6)] + \\ 
 &  & (p_3 - p_6) [(p_1 - p_4) - (p_2 - p_5)].
\end{eqnarray*}
Notice that $p_{4} = p_{\bar{1}}$, $p_{5} = p_{\bar{2}}$, $p_{6} = p_{\bar{3}}$. Then the 
substitution $q_i = p_i - p_{\bar{i}}$ transforms the last expression for $\Delta'$ into
\begin{eqnarray*}
\Delta' & = & q_1(q_3 - q_2) - q_2(q_1 + q_3) + q_3(q_1 - q_2) - 3 \sum_{i=1}^3 q_i^2 \\
       & = & 2(q_1 q_3 - q_1 q_2 - q_2 q_3) - 3 \sum_{i=1}^3 q_i^2.
\end{eqnarray*}
The substitutions 
\begin{eqnarray*}
(q_1 + q_2)^2 = q_1^2 + 2q_1q_2 + q_2^2 \\
(q_2 + q_3)^2 = q_2^2 + 2q_2q_3 + q_3^2 \\
(q_1 - q_3)^2 = q_1^2 - 2q_1q_3 + q_3^2
\end{eqnarray*}
lead to 
\begin{eqnarray*}
\Delta' & = & - \left[(q_1 + q_2)^2 + (q_2 + q_3)^2 + (q_1 - q_3)^2 + \sum_{i=1}^3 q_i^2 \right] \\
        & = & - \Delta
\end{eqnarray*}
after some algebra. 
\end{proof}

The compact formula for $\left<d\right>$ above yields an alternative proof of Property \ref{prop:global_upper_bound} when 
$n=3$. Notice that  
\begin{eqnarray*}\averageswapdistance = \frac{3-\Delta}{2} \leq \frac{3}{2}\end{eqnarray*} because $\Delta \geq 0$ (recall Equation \ref{eq:fancy_Delta}).

\section{Null hypotheses}

\label{app:null_hypotheses}

Here we derive the expectation of the swap distance score and entropy under certain null hypotheses.
 
\label{subsec:random_baselines}

\subsection{Die rolling experiment}

We define $\multinomial$ as a multinomial distribution where $F$ is the number of trials, $N$ is the number of bins and the probability of each bin is $1/N$. We define $f$ as the vector $f = f_1,\dots,f_i,\dots,f_N$ where $f_i$ is the number of times bin $i$ has been chosen.

\begin{lemma} 
\label{lem:help_multinomial}
When $f \sim \multinomial$, the expected value of the product $f_i f_j$ when $i \neq j$ is 
\begin{eqnarray*}\mathbb{E}_{i \neq j}\left[ f_i f_j \middle| f \sim \multinomial \right] = \frac{F (F - 1)}{N^2}.\end{eqnarray*}
\end{lemma}
\begin{proof}
By definition
\begin{eqnarray*}f_i = \sum_{k=1}^F a_{ik},\end{eqnarray*}
where $a_{ik}$ is a Bernoulli random variable that indicates if the $i$-th bin has received a ball in the $k$-th trial. Then
\begin{eqnarray*}
f_i f_j = \sum_{k=1}^F a_{ik} \sum_{l=1}^F a_{jl}.
\end{eqnarray*}
By the linearity of expectation,
\begin{eqnarray*}
\mathbb{E}_{i\neq j}[f_i f_j] = \sum_{k=1}^F \sum_{l=1}^F \mathbb{E}_{i\neq j}[a_{ik}a_{jl}].
\end{eqnarray*}
When $l=k$, $\mathbb{E}_{i\neq j}[a_{ik}a_{jl}] = 0$. When $l \neq k$, $a_{ik}a_{jl}= 1$ with probability $1/N^2$. 
Hence
\begin{eqnarray*}
\mathbb{E}_{i\neq j}[f_i f_j] & = & \sum_{k=1}^F \sum_{\substack{l=1\\l \neq k}}^F \frac{1}{N^2} \\
    & = & \frac{F(F-1)}{N^2}.
\end{eqnarray*}
\end{proof}

\begin{property}
The expected value of $\averageswapdistance$ in a permutohedron of order $n$ when $f$ is generated by rolling a die $F$ times ($F > 0$) is  
\begin{equation*}
\averageswapdistance_{dr} = \mathbb{E}[\averageswapdistance] = \frac{F-1}{F} \frac{d_{max}}{2}.
\end{equation*}
\end{property}
\begin{proof}  
By the linearity of expectation over the definition of $\averageswapdistance$ (Equation \ref{eq:average_swap_distance}) with $N = n!$ and $p_i = f_i/T$ and thanks to $d_{ii} = 0$, we obtain
\begin{eqnarray*} 
\mathbb{E}[\averageswapdistance] &= \frac{1}{F^2} \sum_{i=1}^{N} \sum_{\substack{j=1\\j\neq i}}^{N} d_{ij} \mathbb{E}_{i \neq j}\left[f_i f_j| f \sim \multinomial\right].
\end{eqnarray*}
By invoking Lemma \ref{lem:help_multinomial}, we obtain 
\begin{eqnarray*}
\mathbb{E}[\averageswapdistance] &= & \frac{1}{F^2} \sum_{i=1}^{N} \sum_{\substack{j=1\\j\neq i}}^{N} d_{ij} \frac{F (F - 1)}{N^2}\\
&= & \frac{F - 1}{F} \frac{1}{N^2}\sum_{i=1}^{N} \sum_{\substack{j=1\\j\neq i}}^{N} d_{ij} \\
&= & \frac{F - 1}{F} \averageswapdistance_{op} .
\end{eqnarray*}
Thanks to Property \ref{prop:mean_topological_distance},
we finally obtain 
\begin{eqnarray*}
\mathbb{E}[\averageswapdistance] = \frac{F - 1}{F} \frac{d_{max}}{2}.
\end{eqnarray*}
\end{proof}

The following property gives an expression for the expectation of $H$ under the die rolling null hypothesis. 

\begin{property}
When $f \sim \multinomial$ and $H$ is computed with the plug-in estimator, the expected value of $H$ is
\begin{eqnarray*}H_{dr} = \mathbb{E}[H] =  \log F - \frac{N}{F} \sum_{a=0}^F p(F, 1/N, a) a \log a,\end{eqnarray*}
where 
\begin{eqnarray*}
p(F, q, a) & = & {F \choose a} q^a \left(1 - q \right)^{F - a}.\\
q & = & \frac{1}{N}       
\end{eqnarray*}
\end{property}
\begin{proof}
Thanks to \ref{eq:entropy_plug-in_estimator_expanded},
\begin{eqnarray*}
\mathbb{E}[H] & = & \log F - \frac{1}{F}\sum_{i=1}^{N} \mathbb{E}[f_i \log f_i] \\
              & = & \log F - \frac{N}{F}\mathbb{E}[f_i \log f_i].
\end{eqnarray*}
$\mathbb{E}[f_i \log f_i]$ is the expected value of $f_i \log f_i$ knowing that $f_i$ follows a binomial distribution with parameters $F$ and $1/N$. 
Then 
\begin{eqnarray*}
\mathbb{E}[f_i \log f_i] = \sum_{a = 2}^F p(F, q, a) a \log a,
\end{eqnarray*}
where $p(F, q, a)$ is the probability of $a$ successes over $F$ trials with a probability of success of $q = 1/N$.
\end{proof}

\subsection{Random permutation}

\begin{property}
\label{prop:expected_average_swap_distance_random_permutation}
The expected value of $\averageswapdistance$ when the original $p$ is replaced by a uniformly random permutation of it is given by the expression
\begin{equation*}
\averageswapdistance_{rp} = \mathbb{E}[\averageswapdistance] = \averageswapdistance_{up} \dominanceindex,
\end{equation*}
where $\dominanceindex$ is the dominance index and $\averageswapdistance_{up}$ is the average distance in the permutohedron over unordered pairs.
\end{property}

\begin{proof}
Recall that $N$ is the number of vertices of the permutohedron of order $n$ ($N = n!$).
We define $\pi$ as a permutation function, namely a one-to-one mapping between natural numbers in $[1, N]$ and natural numbers in $[1, N]$. 
By definition of $\averageswapdistance$ and by the linearity of expectation
\begin{eqnarray*}
    \mathbb{E}[\averageswapdistance]  & = & \sum_{i=1}^{N}\sum_{j=1}^{N}d_{ij} \mathbb{E}[p_i p_j],
\end{eqnarray*}
where
\begin{eqnarray*}
    \mathbb{E}[p_i p_j] = \frac{1}{N!} \sum_{\pi} p_{\pi(i)} p_{\pi(j)}.
\end{eqnarray*} 
If $i = j$, $d_{ij} = 0$, which implies that $d_{ij} \mathbb{E}[p_i p_j] = 0$.
If $i \neq j$, then
\begin{eqnarray*}
    \sum_{\pi} p_{\pi(i)} p_{\pi(j)} &=& 2 \cdot \frac{N!}{N(N-1)} \sum_{i=1}^{N} p_i \sum_{j=i+1}^{N} p_j\\
     &=& 2 \cdot \frac{N!}{N(N-1)} \frac{1}{2} \left( \sum_{i=1}^{N} \sum_{j=1}^{N} p_i p_j - \sum_{i=1}^{N} p_i^2 \right)\\
     &=& \frac{N!}{N(N-1)} \left( 1 - \sum_{i=1}^{N} p_i^2 \right)\\
     &=& \frac{N!}{N(N-1)} \dominanceindex,\\
\end{eqnarray*}
which in turn gives
\begin{eqnarray*}
    \mathbb{E}[p_i p_j] &=& \frac{1}{N!} \sum_{\pi} p_{\pi(i)} p_{\pi(j)}\\
    &=& \frac{1}{N(N-1)} \dominanceindex.
\end{eqnarray*} 
Finally, 
\begin{eqnarray*}
     \mathbb{E}[\averageswapdistance]  &=& \sum_{i=1}^{N}\sum_{j=1}^{N}d_{ij} \mathbb{E}[p_i p_j]\\
      &=& \frac{1}{N(N-1)} \dominanceindex \sum_{i=1}^{N}\sum_{j=1}^{N}d_{ij}\\
      &=& \averageswapdistance_{up} \dominanceindex.
\end{eqnarray*}
\end{proof}

\subsection{The $p$-value of the random permutation test}

\label{subsec:p_value_permutation_test}

We aim to find lower bounds of $\mathbb{P}$, the left $p$-value of the permutation test, to understand the statistical power of the permutation test, namely its {\em a priori} capacity to reject the null hypothesis. By definition, 
\begin{eqnarray*}
\mathbb{P}\geq \frac{1}{N!}.
\end{eqnarray*}
However, we will show that $\mathbb{P}$ is predetermined to be ``large'' specially when $n=3$. 

Given a vector of probabilities $p$, every permutation can be represented by a 0-1 matrix $\Pi$ with exactly one 1 per row and column, where the permuted vector is $\Pi p$. We have already seen that $\averageswapdistance$ can be expressed in matrix form as in Equation \ref{eq:average_swap_distance_in_matrix_algebra}.
We say that a permutation $\Pi$ is ``unchanging" with respect to $p$ if
\begin{eqnarray*} 
p^T\Pi^T D \Pi p = p^T D p,
\end{eqnarray*}
where $D$ is the distance matrix of the permutohedron. Equivalently, we can also define the notion of unchanging permutation using a bijective function $\pi: [1,N] \rightarrow [1,N]$. Then $\pi$  is unchanging with respect to $p$ if 
\begin{eqnarray*}
\averageswapdistance = \averageswapdistance_\pi = \sum_{i=1}^N \sum_{j=1}^N d_{\pi(i)\pi(j)} p_i p_j.
\end{eqnarray*}

We define $\mathbb{P}_=$ as the proportion of permutations with same $\averageswapdistance$ as in the original vector $p$, namely $\mathbb{P}_=$ is the proportion of unchanging permutations with respect to $p$. 

We define a mask as a subset of vertices of the permutohedron that have non-zero probability. We represent that mask as a binary vector $s = (s_1,\dots,s_i,\dots, s_N) \in \{0,1\}^N$ with exactly $m$ ones. A vector of probabilities $p$ is said to follow the mask when $p_i = 0$ if and only if $s_i = 0$. 
Given a mask, $\mathbb{P}_m$ is the proportion of permutations $\Pi$ 
such that $\Pi$ is unchanging with respect to every vector $p$ that follows the mask.

We define an equivalence relation between masks. Two masks are equivalent if and only if a graph automorphism of the permutohedron produces one mask from the other. Any two masks in the same equivalence class yield the same $\mathbb{P}_m$. Distinct equivalence classes may yield the same $\mathbb{P}_m$. 

We define $\mathbb{P}_A$, as the proportion of permutations of the vertices that are automorphisms of the permutohedron graph. Recall that an automorphism of a graph is a permutation of its vertices which brings the graph into itself. 

To recap, $\mathbb{P}_=$ gives a lower bound to $\mathbb{P}$ based on both the vector $p$ and the permutohedron graph. $\mathbb{P}_m$ gives a lower bound to $\mathbb{P}$ based on both the placement of the non-zero elements of $p$ and the permutohedron graph. $\mathbb{P}_A$ gives a lower bound to $\mathbb{P}$ based just on the permutohedron graph. Notice that only $\mathbb{P}_=$ takes the actual vector $p$ of probabilities as input.

The following property explains why the permutation test lacks statistical power when $n=3$ by showing that $\mathbb{P} \geq \frac{1}{60}$.
\begin{property}
\label{prop:statistical_power}
The left $p$-value of the permutation test obeys the following chain of inequalities
\begin{eqnarray*}
\mathbb{P} \geq \mathbb{P}_= \geq \mathbb{P}_m \geq \mathbb{P}_A \geq \frac{1}{N!}.
\end{eqnarray*}
When $n = 3$,
\begin{eqnarray*}
\mathbb{P}_A = \frac{1}{60}.\end{eqnarray*}
When $n = 4$,
\begin{eqnarray*}
\mathbb{P}_A = \frac{2}{23!} \approx 7.74 \cdot 10^{-23}.\end{eqnarray*} 
Let $m$ be the number of non-zero probability orders in $p$. The value of $\mathbb{P}_m$ as a function of $m$ is shown in Table \ref{tab:same_average_swap_distance} for $n = 3$ or $n = 4$ with $m \leq 3$ and in 
Table \ref{tab:same_average_swap_distance_4_4} for $n = m = 4$.
\end{property}

\begin{proof}
\begin{eqnarray*}
\mathbb{P}_A = \frac{A}{N!},\end{eqnarray*}
where $A$ is the number of automorphisms. Here we compute $A$ using the BLISS algorithm \citep{Junttila2007a}. \footnote{We used the implementation available in the library {\tt igraph} for {\tt R}, \url{https://igraph.org/r/doc/automorphisms.html}.}
When $n=3$, $A=12$ and then 
\begin{eqnarray*}
\mathbb{P}_A = \frac{1}{60}.\end{eqnarray*}
Alternatively, $A=12$ can be derived as the number of symmetries of a hexagon. 
Without loss of generality, suppose that the permutohedron of order 3 is plotted as a regular hexagon in a two dimensional space (Figure \ref{fig:permutohedron_3_and_4}). That hexagon has six rotational symmetries and six reflection symmetries, which constitute the dihedral group $D_6$ \cite[Chapter 2, p. 25]{Allan1984a}.
When $n=4$, $A=48$ and then the BLISS algorithm yields
\begin{eqnarray*}
\mathbb{P}_A = \frac{2}{23!}.\end{eqnarray*}


Given a value of $m$, $\mathbb{P}_m$ was calculated by means of a brute force method over the space of all masks with some shortcuts. The main shortcut is that the exploration is carried out in fact on the sequences of $m$ distinct vertices of the permutohedron. Each of those sequences encodes both the mask and the permutation of the non-zero probabilities. The exploration over these sequences takes advantage of the fact that the permutations that only exchange zero probabilities are redundant for computing the proportion that defines $\mathbb{P}_m$. The method is described in Algorithm \ref{alg:algorithm}.

\begin{algorithm}
\caption{\label{alg:algorithm} Sketch of the algorithm to compute $\mathbb{P}_m$ for each class given $m$.}
\begin{enumerate}
\item
\label{item:sequence_generation} 
Generate all the sequences of $m$ distinct vertices of the permutohedron where the 1st vertex is the same (the latter is a trick to save computation time). 
\item
For each sequence $\sigma$, set $g(\sigma)$ to 1.
\item 
For every unordered pair of distinct sequences $\sigma_1$, $\sigma_2$
  \begin{enumerate}
  \item 
  Given some sequence $\sigma$, $\averageswapdistance_{\sigma}$ is defined as
  \begin{eqnarray*}
  \averageswapdistance_{\sigma} = \sum_{i=1}^m \sum_{j=1}^m d_{\sigma(i)\sigma(j)}
  p_i p_j,
  \end{eqnarray*}
  where $\sigma(i)$ is the $i$-th vertex in $\sigma$.
  \item
  Compute $\averageswapdistance_{\sigma_1}$ and $\averageswapdistance_{\sigma_2}$.
  \item 
  Compare the coefficients of the polynomial for $\averageswapdistance_{\sigma_1}$ against the coefficients of the polynomial for $\averageswapdistance_{\sigma_2}$.
  \item 
  If $\averageswapdistance_{\sigma_1} = \averageswapdistance_{\sigma_2}$ then add 1 to $g(\sigma_1)$ and to $g(\sigma_2)$.
  \end{enumerate}
\item 
$\sigma_1$ and $\sigma_2$ belong to the same class if and only if 
$\averageswapdistance_{\sigma_1} = \averageswapdistance_{\sigma_2}$.
\item  
For each $\sigma$, $\mathbb{P}_m = g(\sigma)/G$ where $G$ is the number of sequences in Step \ref{item:sequence_generation}.
\end{enumerate}
\end{algorithm}

Each class can be identified by a member of its class, namely $m$ distinct vertices that have non-zero probability. However, for $n = 3$ or $n=4$ with $m \leq 3$, each class can be represented compactly by the multiset of pairwise swap distances between the vertices of the permutohedron that have non-zero frequency. Table \ref{tab:same_average_swap_distance} shows these multisets and their corresponding value of $\mathbb{P}_m$ as a function of $m$ for $n = 3$ or $n=4$ with $m \leq 3$.

\begin{figure}
\centering
\includegraphics[width = 0.7 \textwidth]{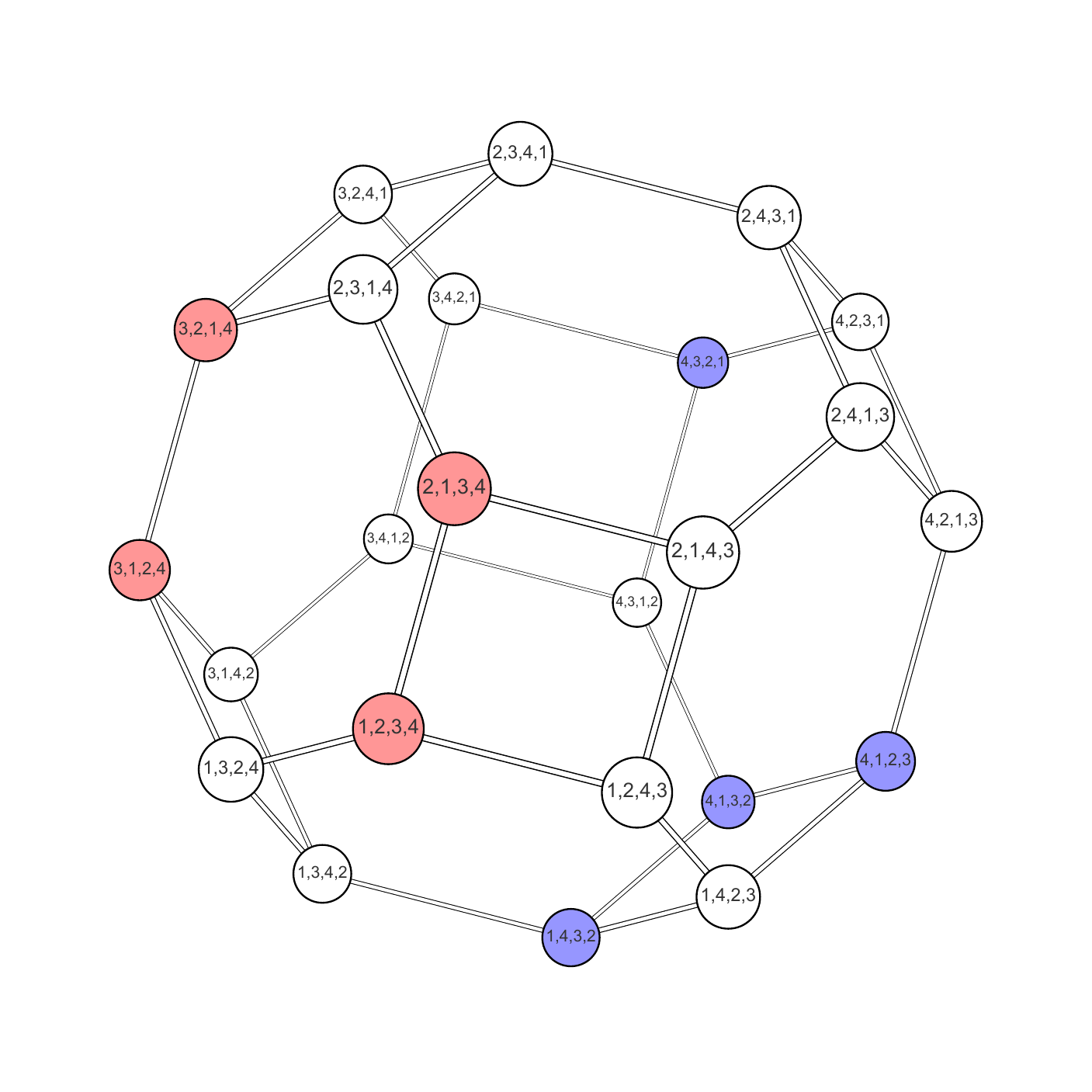}
\caption{\label{fig:different} The permutohedron of order 4 showing two sets of four vertices (red and blue) that belong to different classes of masks. The red group has $\mathbb{P}_m = 4/10626$. The blue group has $\mathbb{P}_m = 8/10626$. }
\end{figure}

When $n=4$ and $m=4$, the multiset of pairwise swap distances fails to characterize a class of vectors with same $\mathbb{P}_m$. Figure \ref{fig:different} shows two distinct groups of vertices that have the same multiset of swap distances, i.e. $\{1,1,2,2,3,3\}$, but fail to yield the same $\mathbb{P}_m$. 
Table \ref{tab:same_average_swap_distance_4_4} shows $\mathbb{P}_m$ and a representative of one the classes with same $\mathbb{P}_m$ for $n = m = 4$.

By definition, $\mathbb{P} \geq \mathbb{P}_=$.
We have $\mathbb{P}_= \geq \mathbb{P}_m$ because the former comprises a wider set of  permutations. 
We have $\mathbb{P}_m \geq \mathbb{P}_A$ instead of $\mathbb{P}_m = \mathbb{P}_A$ because 
$\mathbb{P}_A$ does not take into account that vertices that have zero probability are another source of symmetries. Indeed, Table \ref{tab:same_average_swap_distance} shows that $\mathbb{P}_A = \mathbb{P}_m$ requires $m = N$ or $m = N - 1$ when $n=3$.
\end{proof}

Figure \ref{fig:p_value_lower_bound} shows a negative correlation between $m$ and $\mathbb{P}_m$. Namely, the statistical power of the random permutation test decreases as $m$ decreases, which in turn can be caused by swap distance minimization or entropy minimization.

\begin{figure}
\centering
\includegraphics[width = 0.9 \textwidth]{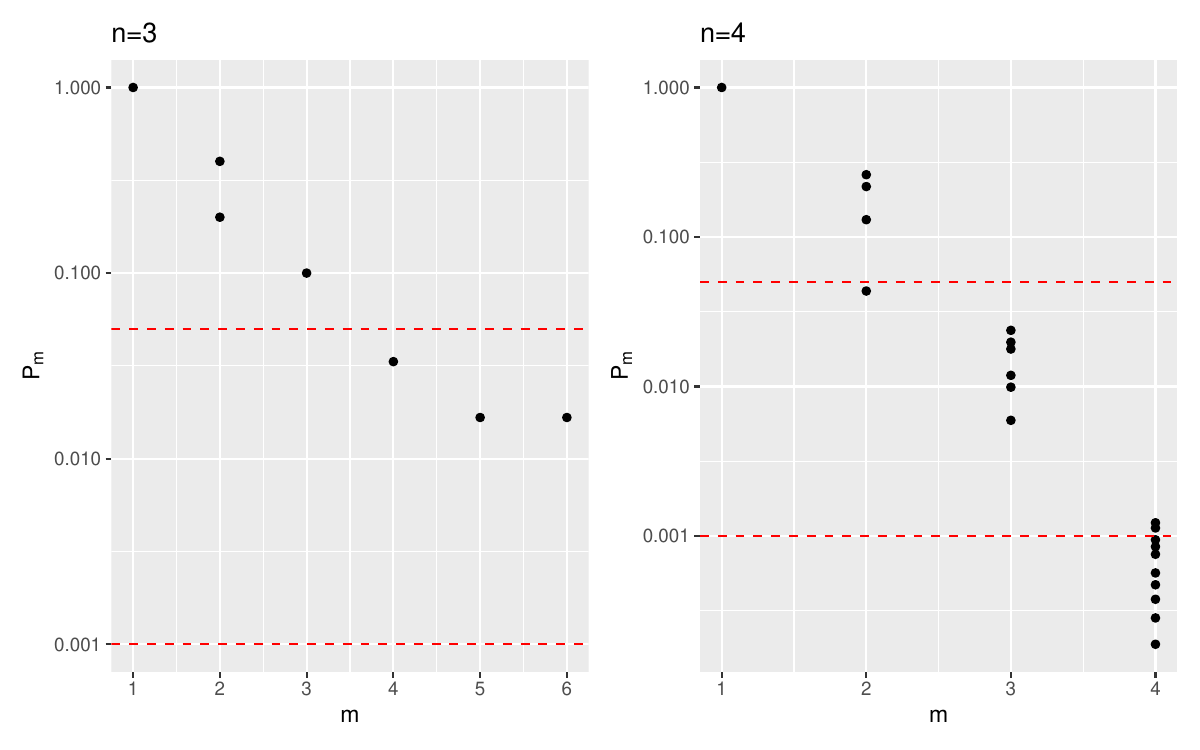}
\caption{\label{fig:p_value_lower_bound} $\mathbb{P}_m$ versus $m$, the number of non-zero probabilities, for $n=3$ (left) and $n=4$ (right). Every point corresponds to a class of probability vectors where $\mathbb{P}_m$ is the same for every member of the class. The red dashed lines indicate the significance levels of $0.05$ and $0.001$, respectively. Data are borrowed from Tables \ref{tab:same_average_swap_distance} and \ref{tab:same_average_swap_distance_4_4}. }
\end{figure}

\begin{table}
\caption{\label{tab:same_average_swap_distance} $\mathbb{P}_m$ as a function of $m$ for $1 \leq m \leq N$ for each class of binary masks. Every class is identified by a multiset of pairwise swap distances for the range of values of $n$ and $m$ shown in this table. $-$ indicates that all classes of masks have the same $\mathbb{P}_m$. Often, distinct classes (represented by distinct multisets) have the same $\mathbb{P}_m$. }
\centering
\begin{tabular}{llll}
\hline
$n$ & \multicolumn{1}{l}{$m$}                  & \multicolumn{1}{l}{$\mathbb{P}_m$} & \multicolumn{1}{l}{Multisets of swap distances}                          \\ \hline
3 & \multicolumn{1}{l}{1}                  & \multicolumn{1}{l}{1}                                & \multicolumn{1}{l}{-}                                              \\ 
& \multicolumn{1}{l}{2} & \multicolumn{1}{l}{1/5}                              & \multicolumn{1}{l}{$\{3\}$.}      \\ 
& \multicolumn{1}{l}{}                   & \multicolumn{1}{l}{2/5}                              & \multicolumn{1}{l}{$\{1\}$, $\{2\}$.} \\ 
& \multicolumn{1}{l}{3}                  & \multicolumn{1}{l}{1/10}                             & \multicolumn{1}{l}{-}                                              \\ 
& \multicolumn{1}{l}{4}                  & \multicolumn{1}{l}{1/30}                             & \multicolumn{1}{l}{-}                                              \\ 
& \multicolumn{1}{l}{5}                  & \multicolumn{1}{l}{1/60}                             & \multicolumn{1}{l}{-}                                              \\ 
& \multicolumn{1}{l}{6}                  & \multicolumn{1}{l}{1/60}                             & \multicolumn{1}{l}{-}                                              \\ 

4 & 1                  & 1           & -                                                  \\ 
&2 & 1/23        & $\{6\}$.          \\ 
&  & 3/23        & $\{1\}$, $\{5\}$     \\ 
&  & 5/23        & $\{2\}$, $\{4\}$     \\ 
&  & 6/23        & $\{3\}$          \\ 
&3 & 3/506       & $\{1, 5, 6\}$                              \\ 
&  & 5/506       & $\{2, 4, 6\}$                              \\ 
&  & 6/506       & $\{1, 1, 2\}$, $\{1, 4, 5\}$, $\{2, 5, 5\}$, $\{3, 3, 6\}$ \\ 
&  & 9/506       & $\{1, 2, 3\}$, $\{1, 3, 4\}$, $\{2, 3, 5\}$, $\{3, 4, 5\}$ \\ 
&  & 10/506      & $\{2, 2, 2\}$, $\{2, 2, 4\}$, $\{2, 4, 4\}$, $\{4, 4, 4\}$ \\ 
&  & 12/506      & $\{2, 3, 3\}$, $\{3, 3, 4\}$                   \\ 
\hline
\end{tabular}
\end{table}

\begin{table}
\caption{\label{tab:same_average_swap_distance_4_4} $\mathbb{P}_m$ for the classes of binary masks that result with $n=m=4$. For each distinct $\mathbb{P}_m$, the representative of only one of the classes is shown as a set of $m$ vertices of the permutohedron of order 4. }
\centering
\begin{tabular}{llll}
\hline
$n$ & \multicolumn{1}{l}{$m$}                  & \multicolumn{1}{l}{$\mathbb{P}_m$} & \multicolumn{1}{l}{Representative of the class}                          \\ \hline
4 & 4 & 2/10626  & 1,2,3,4; 2,1,3,4; 1,2,4,3; 2,1,4,3 \\
& & 3/10626  & 1,2,3,4; 2,1,3,4; 4,3,1,2; 4,3,2,1 \\
& & 4/10626  & 1,2,3,4; 2,1,3,4; 3,1,2,4; 3,2,1,4 \\
& & 5/10626  & 1,2,3,4; 2,3,1,4; 4,1,3,2; 4,3,2,1 \\
& & 6/10626  & 1,2,3,4; 2,1,3,4; 1,3,2,4; 1,2,4,3 \\
& & 8/10626  & 1,2,3,4; 2,1,3,4; 1,3,2,4; 1,4,2,3 \\
& & 9/10626  & 1,2,3,4; 2,1,3,4; 3,1,2,4; 4,2,1,3 \\
& & 10/10626 & 1,2,3,4; 2,1,3,4; 1,3,2,4; 2,3,1,4 \\
& & 12/10626 & 1,2,3,4; 2,1,3,4; 1,3,2,4; 2,3,4,1 \\
& & 13/10626 & 1,2,3,4; 2,1,3,4; 3,1,2,4; 2,4,1,3 \\
\hline
\end{tabular}
\end{table}